\documentclass{article}
\usepackage[top=1in, bottom=1in, left=1in, right=1in]{geometry}

\usepackage{algorithm}
\usepackage{algorithmic}
\usepackage{stfloats}
\usepackage{thmtools}
\usepackage{thm-restate}

\usepackage{url}
\usepackage[hidelinks]{hyperref}
\usepackage{amsmath,amsthm,amssymb,amsbsy}
\usepackage{paralist}
\usepackage{xcolor}
\usepackage{color}
\usepackage{graphicx}
\graphicspath{{./figs/}}
\usepackage{algorithm}
\usepackage{algorithmic}
\usepackage{comment}

\usepackage{fancyhdr}
\usepackage{cite}
\usepackage{cleveref}

\newtheorem{thm}{Theorem}
\newtheorem{cor}{Corollary}
\newtheorem{prop}{Proposition}
\newtheorem{defi}{Definition}

\newtheorem{lem}{Lemma}

\theoremstyle{remark}

\newcommand{\R}{\mathbb{R}}

\newcommand{\e}{\begin{equation}}
\newcommand{\ee}{\end{equation}}
\newcommand{\en}{\begin{equation*}}
\newcommand{\een}{\end{equation*}}
\newcommand{\eqn}{\begin{eqnarray}}
\newcommand{\eeqn}{\end{eqnarray}}
\newcommand{\bmat}{\begin{bmatrix}}
\newcommand{\emat}{\end{bmatrix}}

\newcommand{\dif}{\operatorname{d}}

\newcommand{\vct}[1]{\boldsymbol{#1}}
\newcommand{\mtx}[1]{\boldsymbol{#1}}

\newcommand{\T}{\mathrm{T}}

\newcommand{\trace}{\operatorname{trace}}

\newcommand{\dist}{\operatorname{dist}}
\def \lg        {\langle}
\def \rg        {\rangle}

\newcommand{\domain}{\operatorname{dom}}

\DeclareMathOperator*{\minimize}{\text{minimize}}

\DeclareMathOperator*{\argmin}{\text{arg~min}}
\DeclareMathOperator*{\argmax}{\text{arg~max}}
\def \st {\operatorname*{subject\ to\ }}

\newcommand{\calC}{\mathcal{C}}
\newcommand{\calD}{\mathcal{D}}

\newcommand{\calO}{\mathcal{O}}

\newcommand{\ve}{\vct{e}}

\newcommand{\vu}{\vct{u}}
\newcommand{\vv}{\vct{v}}

\newcommand{\vz}{\vct{z}}
\newcommand{\vzero}{\vct{0}}

\newcommand{\mA}{\mtx{A}}
\newcommand{\mB}{\mtx{B}}

\newcommand{\mD}{\mtx{D}}

\newcommand{\mG}{\mtx{G}}
\newcommand{\mH}{\mtx{H}}

\newcommand{\mM}{\mtx{M}}

\newcommand{\mS}{\mtx{S}}

\newcommand{\mU}{\mtx{U}}
\newcommand{\mV}{\mtx{V}}
\newcommand{\mW}{\mtx{W}}
\newcommand{\mX}{\mtx{X}}
\newcommand{\mY}{\mtx{Y}}

\newcommand{\mLambda}{\mtx{\Lambda}}

\newcommand{\mPhi}{\mtx{\Phi}}

\newcommand{\mId}{{\bf I}}

\newcommand{\mzero}{{\bf 0}}

\graphicspath{{./figs/}}

\newlength{\imgwidth}
\newboolean{twoColVersion}
\setboolean{twoColVersion}{false}
\newcommand{\twoCol}[2]{\ifthenelse{\boolean{twoColVersion}} {#1} {#2} }

\usepackage[framemethod=tikz]{mdframed}

\title{Dropping Symmetry for Fast Symmetric Nonnegative Matrix Factorization}
\author{Zhihui Zhu, Xiao Li, Kai Liu, and Qiuwei Li 
	\thanks{The first and second authors contributed equally to this paper. ~Z. Zhu is with the Mathematical Institute for Data Science, Johns Hopkins University, Baltimore, MD, USA (Email: zzhu29@jhu.edu). ~Xiao Li is with the Department of Electronic Engineering, The Chinese University of Hong Kong, Shatin, NT, Hong Kong (E-mail: xli@ee.cuhk.edu.hk).  ~Kai Liu is with Department of Computer Science, Colorado School of Mines, Golden, CO, USA (E-mail: kaliu@mines.edu). ~ Qiuwei Li is with Department of Electrical Engineering,
		Colorado School of Mines, Golden, CO, USA (E-mail: qiuli@mines.edu).
	}
}

\begin{document}
\maketitle
	\begin{abstract}
	
	Symmetric nonnegative matrix factorization (NMF)---a special but important class of the general NMF---is demonstrated to be useful for data analysis and in particular for various
	clustering tasks. Unfortunately, designing fast algorithms for Symmetric NMF is not as easy as for the nonsymmetric counterpart, the later admitting the splitting property that allows  efficient alternating-type algorithms. To overcome this issue, we transfer the symmetric NMF to a nonsymmetric one, then we can  adopt the idea from the state-of-the-art algorithms for nonsymmetric NMF to design fast algorithms solving symmetric NMF.  We  {rigorously} establish that solving nonsymmetric reformulation returns a solution for symmetric NMF and then apply  fast alternating based algorithms for the corresponding reformulated problem. Furthermore, we show these fast algorithms admit strong convergence guarantee in the sense that the generated sequence is convergent at least at a sublinear rate and it converges globally to a critical point of the symmetric NMF.  We conduct experiments on both synthetic data and image clustering to support our result.		
\end{abstract}

\section{Introduction}

General nonnegative matrix factorization (NMF) is referred to the following problem: Given a matrix $\mY\in \R^{n\times m}$ and a factorization rank $r$, solve
\e
\min_{\mU\in\R^{n\times r},\mV\in\R^{m\times r}}  \frac{1}{2}\|\mY - \mU\mV^\T\|_F^2
,\quad  \st \mU\geq \vzero, \mV \geq \vzero,
\label{eq:NMF}\ee
where $\mU\geq \vzero$ means each element in $\mU$ is nonnegative. NMF has been successfully used in the applications of face feature extraction \cite{lee1999learning,guillamet2002non},  document clustering \cite{shahnaz2006document},  source separation~\cite{ma2014signal} and many others \cite{gillis2014and}. Because of the ubiquitous applications of NMF,  many efficient algorithms have been proposed for solving \eqref{eq:NMF}.  Well-known algorithms include MUA~\cite{lee2001algorithms}, projected gradientd descent \cite{lin2007projected}, alternating nonnegative least squares (ANLS) \cite{kim2008toward}, and hierarchical ALS (HALS) \cite{cichocki2009fast}. In particular, ANLS (which uses the block principal pivoting algorithm to very efficiently solve the nonnegative least squares) and HALS achive the state-of-the-art performance. 

One special but important class of NMF, called symmetric NMF, requires  the two factors $\mU$ and $\mV$ identical, i.e., it factorizes a PSD matrix $\mX\in\R^{n\times n}$ by solving
\e                                                                     
\min_{\mU\in\R^{n\times r}} \frac{1}{2}\|\mX - \mU\mU^\T\|_F^2,\quad \st \mU\geq \vzero.
\label{eq:SNMF}\ee
As a contrast, \eqref{eq:NMF} is referred to as nonsymmetric NMF. Symmetric NMF has its own applications in data analysis, machine learning  and signal processing~\cite{kuang2015symnmf,ding2005equivalence,he2011symmetric}. In particular the symmetric NMF is equivalent to the classical $K$-means  kernel clustering  in \cite{ding2005equivalence}and it is inherently suitable for clustering nonlinearly separable data from a similarity matrix \cite{kuang2015symnmf}.

\begin{figure}[htb!]
	\centering
	\includegraphics[width=0.33\textwidth]{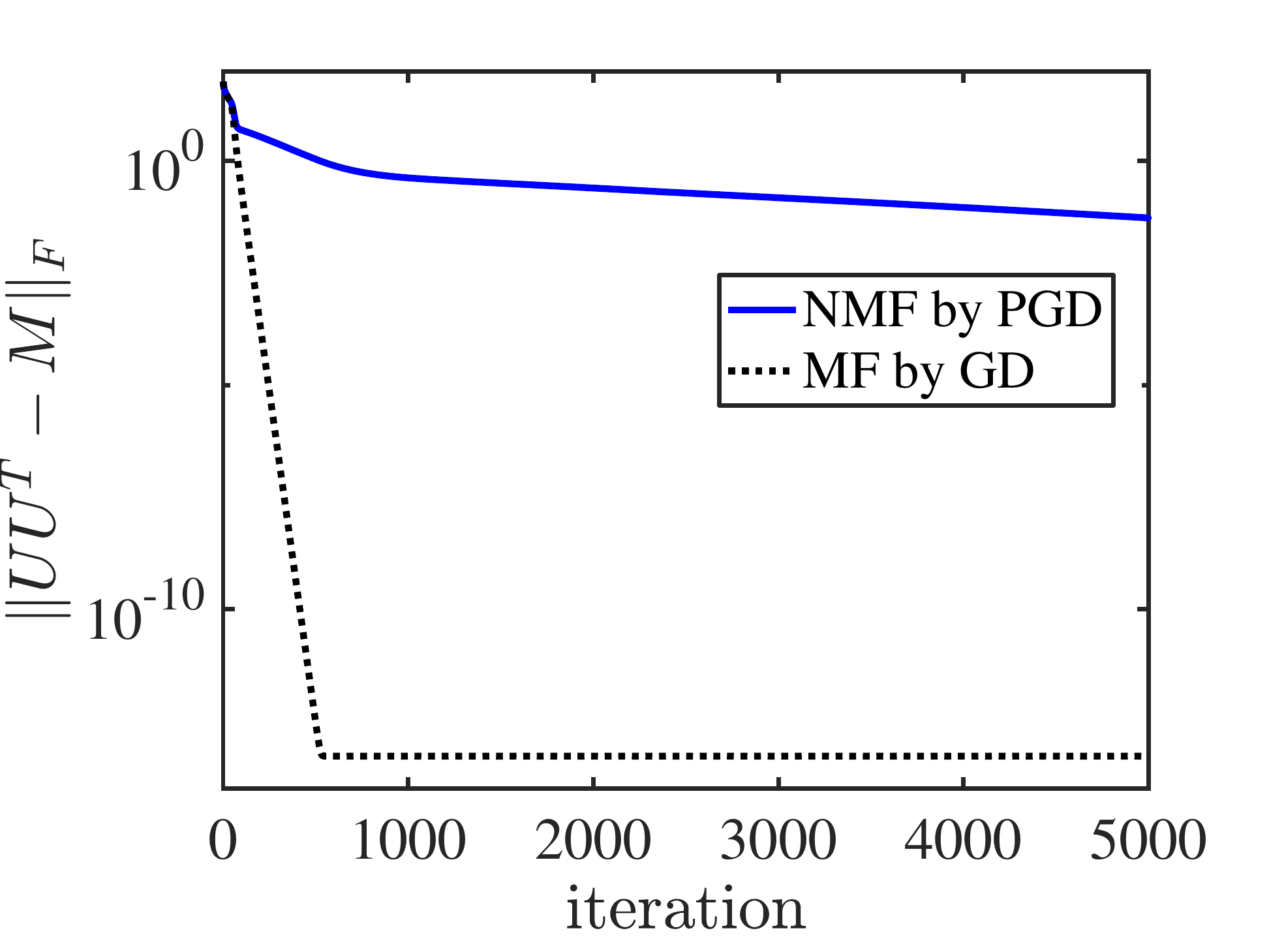}
	\vspace{-.2in}
	\caption{\small Convergence of MF by GD and symmetric NMF by PGD with the same  initialization.}\label{fig:NMFvsMF}
\end{figure}

In the first glance, since \eqref{eq:SNMF}  has only one variable, one may think it is easier to solve \eqref{eq:SNMF} than \eqref{eq:NMF}, or at least \eqref{eq:SNMF}  can be solved by directly utilizing  efficient algorithms developed for nonsymmetric NMF. However, the state-of-the-art alternating based algorithms (such as ANLS and HALS) for nonsymmetric NMF utilize the splitting property of \eqref{eq:NMF} and thus can not be used for \eqref{eq:SNMF}. On the other hand, first-order method such as projected gradient descent (PGD) for solving  \eqref{eq:SNMF} suffers from very slow convergence. As a proof of concept, we show in \Cref{fig:NMFvsMF} the convergence of PGD for solving symmetric NMF and as a comparison, the convergence of gradient descent (GD) for solving a matrix factorization (MF) (i.e., \eqref{eq:SNMF} without the nonnegative constraint) which is proved to admit  linear convergence \cite{tu2015low,zhu2017global}. This phenomenon also appears in nonsymmetric NMF and is the main motivation to have many efficient algorithms such as ANLS and HALS.

\paragraph{Main Contributions} This paper addresses the above issue by considering a simple framework that allows us to design alternating-type algorithms for solving the symmetric NMF, which are similar to alternating minimization algorithms (such as ANLS and HALS) developed for nonsymmetric NMF. The main contributions of this paper are summarized as follows.
\begin{itemize}
	\item Motivated by the splitting property exploited in ANLS and HALS algorithms, we split the bilinear form of $\mU$  into two different factors  and transfer the symmetric NMF into a nonsymmetric one:
	\e
	\min_{\mU,\mV} f(\mU,\mV) = \frac{1}{2}\|\mX - \mU\mV^\T\|_F^2 + \frac{\lambda}{2}\|\mU - \mV\|_F^2, \quad  \st \mU\geq \vzero, \mV \geq \vzero,
	\label{eq:SNMF by reg}\ee
	where the regularizer $\|\mU - \mV\|_F^2$ is introduced to force the two factors identical and $\lambda>0$ is a balancing factor. The first main contribution is to guarantee that any {\em critical point} of \eqref{eq:SNMF by reg} that has bounded energy satisfies $\mU = \mV$ with a sufficiently large $\lambda$. We further show that any local-search algorithm  with a decreasing property is guaranteed to solve \eqref{eq:SNMF} by targeting \eqref{eq:SNMF by reg}. To the best of our knowledge, this is the first work to \emph{rigorously} establish that  symmetric NMF can be efficiently solved by fast alternating-type algorithms.
	\item Our second contribution is to provide  convergence analysis for our proposed alternating-based algorithms solving \eqref{eq:SNMF by reg}. By exploiting the specific structure in \eqref{eq:SNMF by reg}, we show that  our proposed algorithms(without any proximal terms and any additional constraints on $\mU$ and $\mV$ except the nonnegative constraint) is convergent. Moreover, we establish the point-wise global iterates  sequence convergence and show that  the proposed alternating-type algorithms achieve at least a global sublinear convergence rate. Our  sequence convergence result  provides theoretical guarantees for the practical utilization of alternating-based algorithms directly solving \eqref{eq:SNMF by reg} without any proximal terms  or additional constraint on the factors which are usually needed to guarantee the convergence.
\end{itemize}

\paragraph{Related Work} Due to slow convergence of PGD for solving symmetric NMF, different algorithms have been proposed to efficiently solve \eqref{eq:SNMF}, either in a direct way or similar to \eqref{eq:SNMF by reg} by splitting the two factors. Vandaele et al.~ \cite{vandaele2016efficient} proposed an alternating algorithm that 
cyclically optimizes over each element in $\mU$ by solving a nonnegative constrained nonconvex univariate fourth order polynomial minimization.  A  quasi newton second order method was used in \cite{kuang2015symnmf} to directly solve the symmetric NMF optimization problem \eqref{eq:SNMF}. However, both the element-wise updating approach and the second order method are observed to be computationally expensive in large scale applications. We will illustrate this with experiments in \Cref{sec:experiments}.

The idea of solving symmetric NMF by targeting  \eqref{eq:SNMF by reg}  also appears in \cite{kuang2015symnmf}. However, despite an algorithm used for solving \eqref{eq:SNMF by reg},  no other formal guarantee (such as solving \eqref{eq:SNMF by reg} returns a solution of \eqref{eq:SNMF}) was provided in \cite{kuang2015symnmf}. Lu et al. \cite{lu2017nonconvex} considered an alternative problem to \eqref{eq:SNMF by reg} that also enjoys the splitting property and utilized  alternating direction method of multipliers (ADMM) algorithm to tackle the corresponding problem with equality constraint (i.e., $\mU = \mV$). Unlike the sequence convergence guarantee of algorithms solving \eqref{eq:SNMF by reg}, the ADMM  is only guaranteed to have a  subsequence convergence in \cite{lu2017nonconvex} with an additional proximal term\footnote{In $k$-th iteration, a proximal term $\|\mU-\mU_{k-1}\|_F^2$ is added to the objective function when updating $\mU$.} and constraint on the boundedness of columns of $\mU$, rendering the problem hard to solve. 

Finally, our work is also closely related to recent advances in convergence analysis for alternating minimizations. The sequence convergence result for general alternating minimization with an additional proximal term was provided in \cite{attouch2010proximal}.
When specified to NMF, as pointed out in~\cite{huang2016flexible}, with the aid of this additional proximal term (and also an additional constraint to bound the factors), the convergence of ANLS and HALS can be established from~\cite{attouch2010proximal,razaviyayn2013unified}. With similar proximal term and constraint, the subsequence convergence of ADMM for symmetric NMF was obtained in \cite{lu2017nonconvex}. Although the convergence of these algorithms are observed without the proximal term and constraint (which are also not used in practice), these are in general necessary to formally show the convergence of the algorithms. For alternating minimization methods solving \eqref{eq:SNMF by reg}, without any additional constraint, we show the factors are  indeed bounded through the iterations, and without the proximal term, the algorithms admit sufficient decreasing property. These observations then guarantee the sequence convergence of the original algorithms that are used in practice. The  convergence result  for algorithms solving \eqref{eq:SNMF by reg} is not only limited to alternating-type algorithms, though we only consider these as they achieve state-of-the-art performance.

\section{Transfering Symmetric NMF to Nonsymmetric NMF}
\label{sec:SNMF to NMF}
We first rewrite \eqref{eq:SNMF} as
\[
\min_{\mU,\mV} \frac{1}{2}\|\mX - \mU\mV^\T\|_F^2, \st \mU = \mV, \mU\geq 0, \mV \geq 0
\]
and turn to solve the following regularized form:
\e
\min_{\mU,\mV} f(\mU,\mV) = \frac{1}{2}\|\mX - \mU\mV^\T\|_F^2 + \frac{\lambda}{2}\|\mU - \mV\|_F^2 + \delta_+(\mU) + \delta_+(\mV).
\label{eq:SNMF by reg}\ee

Compared  with \eqref{eq:SNMF}, in the first glance, \eqref{eq:SNMF by reg} is slightly more complicated as it has one more variable. However, because of this new variable, $f(\mU,\mV)$ is now strongly convex with respect to either $\mU$ or $\mV$, thought it is still nonconvex in terms of the joint variable $(\mU,\mV)$. Moreover, the two decision variables $\mU$ and $\mV$ in \eqref{eq:SNMF by reg} are well separated, like the case in nonsymmetric NMF. This observation suggests an interesting and useful factor that  \eqref{eq:SNMF by reg} can be solved by tailored state of the art algorithms (such as the alternating minimization type algorithms) develpped for solving the general NMF for solving .

On the other hand, a theoretical question raised in the regularized form \eqref{eq:SNMF by reg} is that we are not guaranteed $\mU = \mV$ and hence solving \eqref{eq:SNMF by reg} is not equivalent to solving \eqref{eq:SNMF}. One of the main contribution is to assure that solving \eqref{eq:SNMF by reg} gives a solution of \eqref{eq:SNMF}.

\begin{thm} Suppose $(\mU^\star,\mV^\star)$ be any critical point of \eqref{eq:SNMF by reg} satisfying $\|\mU^\star\mV^{\star\T}\| < 2 \lambda + \sigma_n(\mX)$, where  $\sigma_n(\cdot)$ denotes the $n$-th largest singular value. Then $\mU^\star = \mV^\star$ and $\mU^\star$ is a critical point of \eqref{eq:SNMF}.
\label{thm:U = V}\end{thm}

\begin{proof}[Proof of \Cref{thm:U = V}] 
We first preset the following useful result, which generalizes the classical result for two PSD matrices.
\begin{lem}\label{lem:trace inequality for one PSD}
For any symmetric $\mA\in\R^{n\times n}$ and PSD matrix $\mB\in\R^{n\times n}$, we have
\[
\sigma_n(\mA)\trace(\mB)\leq \trace\left(\mA\mB\right) \leq \sigma_1(\mA)\trace(\mB),
\]
where $\sigma_i(\mA)$ is the $i$-th largest eigenvalue of $\mA$.
\end{lem}
\begin{proof}[Proof of Lemma~\ref{lem:trace inequality for one PSD}]
Let $\mA = \mPhi_1\mLambda_1\mPhi_1^\T$ and $\mB = \mPhi_2\mLambda_2\mPhi_2^\T$ be the eigendecompositions of $\mA$ and $\mB$, respectively. Here $\mLambda_1$ ($\mLambda_2$) is a diagonal matrix with the eigenvalues of $\mA$ ($\mB$) along its diagonal. We first rewrite $\trace\left(\mA\mB\right)$ as
\begin{align*}
\trace\left(\mA\mB\right) = \trace\left(\mLambda_1\mPhi_1^\T \mPhi_2\mLambda_2\mPhi_2^\T \mPhi_1\right).
\end{align*}
Noting that $\mLambda_1$ is a diagonal matrix and $\mPhi_1^\T \mPhi_2\mLambda_2\mPhi_2^\T \mPhi_1\succeq \mzero$ since $\mLambda_2\succeq \mzero$, we have
\begin{align*}
&\trace\left(\mLambda_1\mPhi_1^\T \mPhi_2\mLambda_2\mPhi_2^\T \mPhi_1\right) \\ &\leq \max_{i}\mLambda_1[i,i]\cdot\trace\left(\mPhi_1^\T \mPhi_2\mLambda_2\mPhi_2^\T \mPhi_1\right) \\&= \sigma_1(\mA)\trace(\mB).
\end{align*}
The other direction follows similarly.
\end{proof}

We now prove \Cref{thm:U = V}. The subdifferential of $f$ is given as follows
\e\begin{split} \label{eq:subdifferential U V}
&\partial_{\mU}f(\mU,\mV) = (\mU\mV^\T - \mX)\mV + \lambda (\mU - \mV) + \partial \delta_+(\mU),\\
&\partial_{\mV}f(\mU,\mV) = (\mU\mV^\T - \mX)^\T\mU - \lambda (\mU - \mV) + \partial \delta_+(\mV),
\end{split}\ee
where $\partial \delta_+(\mU) = \left\{\mG\in\R^{n\times r}:\mG \circ \mU = \vzero, \mG \leq \vzero \right\}$ when $\mU\geq \mzero$  and otherwise $\partial \delta_+(\mU) = \emptyset$. Since $(\mU^\star,\mV^\star)$ is a critical point of \eqref{eq:SNMF by reg}, it satisfies
\begin{align}
(\mU^\star\mV^{\star\T} - \mX)\mV^\star + \lambda (\mU^\star - \mV^\star) + \mG = \vzero,\label{eq:first order U}\\
(\mU^\star\mV^{\star\T} - \mX)^\T\mU^\star - \lambda (\mU^\star - \mV^\star) + \mH = \vzero,\label{eq:first order V}
\end{align}
where $\mG\in  \partial \delta_+(\mU^\star)$ and $\mH\in  \partial \delta_+(\mV^\star)$. Subtracting \eqref{eq:first order V} from \eqref{eq:first order U}, we have
\e
(2\lambda \mId + \mX)(\mU^\star - \mV^\star) =  \mV^\star\mU^{\star\T}\mU^\star - \mU^\star\mV^{\star\T}\mV^\star - \mG + \mH.
\ee
where we utilize the fact that $\mX$ is symmetric, i.e., $\mX = \mX^\T$. Taking the inner product of $\mU^\star - \mV^\star$ with both sides of the above equation gives \e
\langle (\lambda \mId + \mX),(\mU^\star - \mV^\star)(\mU^\star - \mV^\star)^\T\rangle = \langle  \mV^\star\mU^{\star\T}\mU^\star - \mU^\star\mV^{\star\T}\mV^\star - \mG + \mH, \mU^\star - \mV^\star \rangle.
\label{eq: U V innder product}\ee

In what follows, by choosing sufficiently large $\lambda$, we show that $(\mU^\star,\mV^\star)$ satisfying \eqref{eq: U V innder product} must satisfy $\mU^\star = \mV^\star$. To that end, we first provide the lower bound and the upper bound for the LHS and RHS of \eqref{eq: U V innder product}, respectively. Specifically,
\begin{align}
\langle ((2\lambda \mId + \mX),(\mU^\star - \mV^\star)(\mU^\star - \mV^\star)^\T\rangle  \geq \sigma_n((2\lambda \mId + \mX)\|\mU^\star - \mV^\star\|_F^2 = ((2\lambda + \sigma_n( \mX))\|\mU^\star - \mV^\star\|_F^2,
\label{eq:LHS}\end{align}
where the inequality follows from \Cref{lem:trace inequality for one PSD}. On the other hand,
\e\begin{split}
&\langle  \mV^\star\mU^{\star\T}\mU^\star - \mU^\star\mV^{\star\T}\mV^\star - \mG + \mH, \mU^\star - \mV^\star \rangle\\& \leq \langle  \mV^\star\mU^{\star\T}\mU^\star - \mU^\star\mV^{\star\T}\mV^\star, \mU^\star - \mV^\star \rangle\\
&= \left\langle \frac{\mV^\star\mU^{\star\T} + \mU^\star\mV^{\star\T}}{2}, (\mU^\star - \mV^\star)(\mU^\star - \mV^\star)^\T \right\rangle - \frac{1}{2}\left\| \mU^\star\mV^{\star\T}- \mV^\star\mU^{\star\T}\right\|_F^2\\
& \leq \left\langle \frac{\mV^\star\mU^{\star\T} + \mU^\star\mV^{\star\T}}{2}, (\mU^\star - \mV^\star)(\mU^\star - \mV^\star)^\T \right\rangle\\
& \leq \sigma_1\left( \frac{\mV^\star\mU^{\star\T} + \mU^\star\mV^{\star\T}}{2} \right)\|\mU^\star - \mV^\star\|_F^2,
\end{split}\label{eq:RHS}\ee
where the last inequality utilizes \Cref{lem:trace inequality for one PSD} and the first inequality follows because $\mV^\star,\mU^\star\geq \vzero$ indicating that
\[
-\langle\mG , \mU^\star - \mV^\star\rangle \leq  0, \ \ \  \langle\mH , \mU^\star - \mV^\star\rangle \leq  0.
\]

Now plugging \eqref{eq:LHS} and \eqref{eq:RHS} back into \eqref{eq: U V innder product} and utilizing the assumption that $\|\mU^\star\mV^{\star\T}\|_F \leq \alpha$, we have
\begin{align*}
((2\lambda + \sigma_n(\mX)) \|\mU^\star - \mV^\star\|_F^2\leq
\sigma_1\left( \frac{\mV^\star\mU^{\star\T} + \mU^\star\mV^{\star\T}}{2} \right)\|\mU^\star - \mV^\star\|_F^2\leq \alpha\|\mU^\star - \mV^\star\|_F^2,
\end{align*}
which implies that if we choose $2\lambda> \alpha - \sigma_n(\mX)$, then $\mU^\star = \mV^\star$ must hold. Plugging it into \eqref{eq:subdifferential U V} gives 
\[
  \textbf{0}\in (\mU^\star(\mU^\star)^\T - \mX)\mU^\star + \partial \delta_+(\mU^\star),
\]
which implies $\mU^\star$ is a critical point of \eqref{eq:SNMF}.
\end{proof}

%
%
%
%

Towards interpreting \Cref{thm:U = V}, we  note that for any $\lambda>0$, \Cref{thm:U = V} ensures a certain region (whose size depends on $\lambda$) in which each critical point of \eqref{eq:SNMF by reg} has identical factors and also returns a solution for the original symmetric NMF \eqref{eq:SNMF}. This further suggests the opportunity of choosing an appropriate $\lambda$ such that the corresponding region (i.e., all $(\mU,\mV)$ such that $\|\mU\mV^{\T}\| < 2 \lambda + \sigma_n(\mX)$) contains all the possible points that the algorithms will converge to. Towards that end,  next result indicates that for any local search algorithms, if it decreases the objective function, then the iterates are bounded.
\begin{lem} For any local search algorithm solving \eqref{eq:SNMF by reg} with initialization $\mV_0=\mU_0, \mU_0\geq0$, suppose it sequentially decreases the objective value.  Then, for any $k\geq 0$,  the iterate $(\mU_k,\mV_k)$ generated by this algorithm satisfies
	\e\begin{split}
		\|\mU_k\|_F^2+\|\mV_k\|_F^2& \leq \left(\frac{1}{\lambda}+2\sqrt{r}\right)\|\mX-\mU_0\mU_0^\T\|_F^2+2\sqrt{r} \|\mX\|_F:=B_0,\\
		\|\mU_k\mV_k^\T\|_F&\leq\|\mX-\mU_0\mV_0^\T\|_F+\|\mX\|_F.
	\end{split}	\label{eqn:bound}\ee
	
	\label{lem:bound:iterate}
\end{lem} 

\begin{proof}[Proof of \Cref{lem:bound:iterate}]
	By the assumption that the algorithm decreases the objective function, we have
	\begin{align*}
	\frac{1}{2}\left\|\mX - \mU_k \mV_k^{\T }\right \|_F^2 + \frac{\lambda}{2} \left\|\mU_k-\mV_k \right\|_F^2
	\leq \frac{1}{2}\left \|\mX - \mU_0\mU_0^{\T} \right\|_F^2
	\end{align*}
	which further implies that
	\[
	\begin{cases}
	\left\|\mX-\mU_k \mV_k^{\T}\right\|_F \leq \left\| \mX-\mU_0\mU_0^\T\right\|_F\\
	\frac{\lambda}{2} \left(\|\mU_k\|_F^2+\|\mV_k\|_F^2-2|\lg\mU_k\mV_k^\T,\mId_r\rg|\right)\le\frac{\lambda}{2} \left\|\mU_k-\mV_k \right\|_F^2
	\leq \frac{1}{2}\left \|\mX - \mU_0\mU_0^{\T} \right\|_F^2
	\end{cases}
	\]
	where the first line further gives that
	\begin{align}
	\|\mU_k\mV_k^\T\|_F\leq\|\mX-\mU_0\mV_0^\T\|_F+\|\mX\|_F,
	\nonumber
	\end{align}
	while the second line leads to
	\begin{align*}
	\|\mU_k\|_F^2+\|\mV_k\|_F^2
	\leq& \frac{1}{\lambda}\|\mX-\mU_0\mU_0^\T\|_F^2+2\|\mU_k\mV_k^\T\|_F\|\mId_r\|_F\nonumber\\
	=&
	\frac{1}{\lambda}\|\mX-\mU_0\mU_0^\T\|_F^2+2\sqrt{r}\|\mU_k\mV_k^\T\|_F\nonumber\\
	\leq&
	\left(\frac{1}{\lambda}+2\sqrt{r}\right)\|\mX-\mU_0\mU_0^\T\|_F^2+2\sqrt{r} \|\mX\|_F =:B_0
	\end{align*}
	
	
\end{proof}

%
%

There are two interesting facts regarding the iterates can be interpreted from \eqref{eqn:bound}. The first equation of \eqref{eqn:bound} implies that both $\mU_k$ and $\mV_k$ are bounded and the upper bound decays when the  $\lambda$ increases. Specifically, as long as $\lambda$ is not too close to zero, then the RHS in \eqref{eqn:bound} gives a meaningful bound which will be used for the convergence analysis of local search algorithms in next section.
In terms of $\mU_k\mV_k^\T$, the second equation of \eqref{eqn:bound} indicates that it is indeed upper bounded by a quantity that is independent of $\lambda$. This suggests a key result that if  the iterative algorithm is convergent and the iterates $(\mU_k,\mV_k)$  converge to a critical point $(\mU^\star,\mV^\star)$, then  $\mU^\star\mV^{\star T}$ is also bounded, irrespectively the value of $\lambda$. This together with \Cref{thm:U = V} ensures that many local search algorithms can be utilized to find a critical point of \eqref{eq:SNMF} by choosing a sufficiently large $\lambda$.

\begin{thm}
	Choose $\lambda > \frac{1}{2} \left( \|\mX\|_2 + \left\|\mX-\mU_0\mU_0^\T\right\|_F - \sigma_n(\mX)  \right)$ for \eqref{eq:SNMF by reg}. For any local search algorithm solving \eqref{eq:SNMF by reg}  with initialization $\mV_0=\mU_0$, if it sequentially decreases the objective value, is convergent and  converges to a critical point $(\mU^\star,\mV^\star)$ of \eqref{eq:SNMF by reg}, then we have $\mU^\star = \mV^\star$ and that $\mU^\star$ is also a critical point of \eqref{eq:SNMF}.
	\label{thm:critical point U = V}
\end{thm} 

\Cref{thm:critical point U = V} indicates that instead of directly solving the symmetric NMF \eqref{eq:SNMF}, one can turn to solve \eqref{eq:SNMF by reg} with a sufficiently large regularization parameter $\lambda$. The latter is very similar to the nonsymmetric NMF \eqref{eq:NMF} and obeys similar splitting property, which enables us to utilize efficient alternating-type algorithms. In the next section, we propose alternating based algorithms for tackling \eqref{eq:SNMF by reg} provide strong guarantees on the descend property and convergence issue.

\section{Fast Algorithms for Symmetric NMF}
\label{sec:fast algorithms}

In the last section, we have shown that the symmetric NMF \eqref{eq:SNMF} can be transfered to problem \eqref{eq:SNMF by reg}, the latter admitting splitting property which enable us to design alternating-type algorithms to solve symmetric NMF. Specifically, we exploit the splitting property by adopting the main idea in ANLS and HALS  for nonsymmetric NMF to design fast algorithms for \eqref{eq:SNMF by reg}. Moreover, note that the objective function $f$ in \eqref{eq:SNMF by reg} is strongly convex with respect to $\mU$ (or $\mV$) with fixed $\mV$ (or $\mU$) because of the regularized term $\frac{\lambda}{2}\|\mU - \mV\|_F^2$. This together with \Cref{lem:bound:iterate} ensures that  strong descend property and point-wise sequence convergence guarantee of the proposed alternating-type algorithms. With \Cref{thm:critical point U = V}, we are finally guaranteed that the algorithms converge to a critical point of symmetric NMF \eqref{eq:SNMF}.

\subsection{ANLS for symmetric NMF(SymANLS)}
\begin{algorithm}[htb]
	\caption{SymANLS}
	\label{alg:ANLS}
	{\bf Initialization:}  $k=1$ and  $\mU_0 = \mV_0$.
	
	\begin{algorithmic}[1]
		\WHILE{stop criterion not meet}
		\STATE $
		\mU_{k} = \arg\min_{\mV\geq 0}  \frac{1}{2}\|\mX - \mU\mV_{k-1}^\T\|_F^2   + \frac{\lambda}{2}\|\mU - \mV_{k-1}\|_F^2$;
		\STATE 
		$\mV_{k} = \arg\min_{\mU\geq 0}  \frac{1}{2}\|\mX - \mU_{k}\mV^\T\|_F^2 + \frac{\lambda}{2}\|\mU_k - \mV\|_F^2 $;
		\STATE
		$k = k+1$.
		\ENDWHILE
	\end{algorithmic}
	{\bf Output:} factorization $(\mU_k, \mV_k)$.
\end{algorithm}

ANLS is an alternating-type algorithm customized for nonsymmetric NMF \eqref{eq:NMF} and its main idea is that at each time, keep one factor fixed, and update another one via solving a nonnegative constrained least squares.  We use similar idea for solving \eqref{eq:SNMF by reg} and refer the corresponding algorithm as SymANLS. Specifically,  at the $k$-th iteration, SymANLS first updates $\mU_k$ by
\e
\mU_k =\argmin_{\mU\in\R^{n\times r}, \mU\geq 0} \|\mX - \mU\mV_{k-1}^\T\|_F^2   + \frac{\lambda}{2}\|\mU - \mV_{k-1}\|_F^2. 
\label{eq:Uk ANLS}\ee
$\mV_k$ is then updated in a similar way. We depict the  whole procedure of SymANLS in \Cref{alg:ANLS}. With respect to solving the subproblem~\eqref{eq:Uk ANLS}, we first note that there exists a unique minimizer (i.e., $\mU_k$) for \eqref{eq:Uk ANLS} as it involves a strongly objective function as well as a convex feasible region. However, we note that because of the nonnegative constraint, unlike least squares,  in general there is no closed-from solution for \eqref{eq:Uk ANLS} unless $r =1$. Fortunately, there exist many feasible methods to solve the  nonnegative constrained least squares, such as projected gradient descend, active set method and projected Newton's method. Among these methods,  a  block principal pivoting method is remarkably efficient for tackling the subproblem \eqref{eq:Uk ANLS} (and also the one for updating $\mV$) \cite{kim2008toward}.

With the specific structure within \eqref{eq:SNMF by reg} (i.e., its objective function is strongly convex and its feasible region is convex), we first show that  SymANLS monotonically decreases the function value at each iteration, as required in \Cref{thm:critical point U = V}.
\begin{restatable}{lem}{sufficientdecrease}
	\label{lem:sufficient decrease}
	Let $\{(\mU_k,\mV_k)\}$ be the iterates sequence  generated by \Cref{alg:ANLS}.  Then we have 
	\[
	f(\mU_k,\mV_k) - f(\mU_{k+1},\mV_{k+1}) \geq \frac{\lambda}{2}( \|\mU_{k+1}  - \mU_{k}\|_F^2 + \|\mV_{k+1}  - \mV_{k}\|_F^2).
	\]
\end{restatable}
The proof of \Cref{lem:sufficient decrease} is given in \Cref{sec:prf lem sufficient decrease}. We now give the following main convergence guarantee for \Cref{alg:ANLS}.
\begin{thm}[Sequence convergence of \Cref{alg:ANLS}] Let $\{(\mU_k,\mV_k)\}$ be the sequence generated by \Cref{alg:ANLS}. Then 
	\[
	\lim\limits_{k\rightarrow \infty} (\mU_k,\mV_k) = (\mU^\star,\mV^\star),
	\]
	where $(\mU^\star,\mV^\star)$ is a critical point of \eqref{eq:SNMF by reg}. Furthermore the convergence rate is at least sublinear. 
	\label{thm:ANLS}\end{thm}
The proof of \Cref{thm:ANLS} is given in \Cref{sec:prf thm ANLS}. Equipped with all the machinery developed above, the global sublinear sequence convergence of SymANLS to a critical solution of symmetric NMF \eqref{eq:SNMF}  is formally guaranteed in the following result, which is a direct consequence of \Cref{thm:critical point U = V}, \Cref{lem:sufficient decrease} and \Cref{thm:ANLS}. 

\begin{cor}[Convergence of \Cref{alg:ANLS} to a critical point of \eqref{eq:SNMF}]
	Suppose \Cref{alg:ANLS} is initialized  with $\mV_0=\mU_0$. Choose 
	\[
	\lambda > \frac{1}{2} \left( \|\mX\|_2 + \left\|\mX-\mU_0\mU_0^\T\right\|_F - \sigma_n(\mX)  \right).
	\]	
	Let $\{(\mU_k,\mV_k)\}$ be the sequence generated by \Cref{alg:ANLS}.  Then $\{(\mU_k,\mV_k)\}$ is convergent and converges to $(\mU^\star,\mV^\star)$ with $
	\mU^\star = \mV^\star $
	and $\mU^\star$	a critical point of \eqref{eq:SNMF}. Furthermore, the convergence rate is at least sublinear.
	\label{cor:ANLS}\end{cor}
\begin{proof}[Proof of \Cref{cor:ANLS}]
	This follows from \Cref{thm:critical point U = V}, \Cref{lem:sufficient decrease} and \Cref{thm:ANLS}. 
\end{proof}	

\paragraph{Remark.} We  emphasis that the specific structure within \eqref{eq:SNMF by reg} enables \Cref{cor:ANLS} get rid of the assumption on the boundedness of iterates $(\mU_{k},\mV_k)$ and also the requirement of a proximal term, which is usually required for convergence analysis but not necessarily used in practice.  As a contrast and also as pointed out in~\cite{huang2016flexible}, to provide the convergence guarantee for standard ANLS solving nonsymmetric NMF \eqref{eq:NMF}, one needs to modify it by adding an additional proximal term as well as  an additional constraint to make the factors bounded.

\subsection{HALS for symmetric NMF (SymHALS)}
As we stated before,  due to the nonnegative constraint,  there is no closed-from solution for \eqref{eq:Uk ANLS}, although one may utilize some efficient algorithms for solving \eqref{eq:Uk ANLS}. However, there do exist a close-form solution when $r =1$. HALS exploits this observation by splitting the pair of variables $(\mU,\mV)$ into  columns $(\vu_1,\cdots,\vu_r,\vv_1,\cdots,\vv_r)$ and then optimizing over \emph{column by column}. We utilize similar idea for solving \eqref{eq:SNMF by reg}. Specifically, rewrite $\mU\mV^\T = \vu_i\vv_i^\T + \sum_{j\neq i}\vu_j\vv_j^\T$ and denote by
\[
\mX_i = \mX - \sum_{j\neq i}\vu_j\vv_j^\T
\]
the factorization residual $\mX-\mU\mV^\T$ excluding $\vu_i\vv_i^\T$. Now if we minimize the objective function $f$ in \eqref{eq:SNMF by reg} only with respect to $\vu_i$, then it is equivalent to 
\[
\vu_i^\natural = \argmin_{\vu_i\in\R^{n}} \frac{1}{2}\|\mX_i - \vu_i\vv_i^\T\|_F^2 + \frac{\lambda}{2}\|\vu_i - \vv_i\|_2^2 = \max\left(\frac{( \mX_{i} +\lambda \mId)\vv_i}{\|\vv_i\|_2^2+\lambda },0\right).
\]
Similar closed-form solution also holds when optimizing in terms of $\vv_i$. With this observation, we utilize alternating-type minimization that at each time minimizes the objective function in \eqref{eq:SNMF by reg} only with respect to one column in $\mU$ or $\mV$ and denote the corresponding algorithm as SymHALS. We depict SymHALS in \Cref{alg:HALS}, where we use subscript $k$ to denote the $k$-th iteration. Note that to make the presentation easily understood, we directly use $\mX - \sum_{j=1}^{i-1}\vu_j^{k}(\vv_j^k)^\T -  \sum_{j=i+1}^{r}\vu_j^{k-1}(\vv_j^{k-1})^\T$ to update $\mX_i^k$, which is not adopted in practice. Instead, letting $\mX_1^k = \mX - \mU^{k-1} (\mV^{k-1})^\T$, we can then update $\mX_i^k$ with only the computation of $\vu_i^k(\vv_i^k)^\T$ by recursively utilizing the previous one. The detailed information about efficient implementation of SymHALS can be found in  \Cref{alg:HALS 2}. 

\begin{algorithm}[htb]
	\caption{SymHALS}
	\label{alg:HALS}
	{\bf Initialization:}  $\mU_0, \mV_0$, iteration $k=1$. 
	
	\begin{algorithmic}[1]
		\WHILE{stop criterion not meet}
		\FOR{$i=1:r$}
		\STATE $\mX_i^k = \mX - \sum_{j=1}^{i-1}\vu_j^{k}(\vv_j^k)^\T -  \sum_{j=i+1}^{r}\vu_j^{k-1}(\vv_j^{k-1})^\T$;
		\STATE $\vu_i^k  = \argmin_{\vu_i\geq \vzero} \frac{1}{2}\|\mX_{i}^k - \vu_i(\vv_i^{k-1})^\T\|_F^2 + \frac{\lambda}{2}\|\vu_i - \vv_i^{k-1}\|_F^2 = \max\left(\frac{( \mX_{i}^k +\lambda \mId)\vv_i^{k-1}}{\|\vv_i^{k-1}\|_2^2+\lambda },0\right)$;
		\STATE
		$\vv_i^k  = \argmin_{\vv_i\geq \vzero} \frac{1}{2}\|\mX_i^k - \vu_i^k\vv_i^\T\|_F^2  + \frac{\lambda}{2}\|\vu_i^k - \vv_i\|_F^2 = \max\left(\frac{( \mX_{i}^k +\lambda \mId)\vu_i^{k}}{\|\vu_i^{k}\|_2^2+\lambda },0\right);
		$
		\ENDFOR
		\STATE$k = k+1$.
		\ENDWHILE
	\end{algorithmic}
	{\bf Output:} factorization $(\mU_k, \mV_k)$.
\end{algorithm}

\begin{algorithm}[htb]
	\caption{Efficient Implementation of SymHALS}
	\label{alg:HALS 2}
	{\bf Initialization:}  $\mU_0, \mV_0$, iteration $k=1$. 
	
	\begin{algorithmic}[1]
		\STATE precompute residual $\mX_1^k = X - \mU^{k-1}(\mV^{k-1})^\T$.
		\WHILE{stop criterion not meet}
		\FOR{$i=1:r$}
		\STATE $\mX_i^k \leftarrow \mX_i^k + \vu_i^{k-1}(\vv_i^{k-1})^\T$
		\STATE $\vu_i^k  = \argmin_{\vu_i\geq \vzero} \frac{1}{2}\|\mX_{i}^k - \vu_i(\vv_i^{k-1})^\T\|_F^2 + \frac{\lambda}{2}\|\vu_i - \vv_i^{k-1}\|_F^2 = \max\left(\frac{( \mX_{i}^k +\lambda \mId)\vv_i^{k-1}}{\|\vv_i^{k-1}\|_2^2+\lambda },0\right)$;
		\STATE
		$\vv_i^k  = \argmin_{\vv_i\geq \vzero} \frac{1}{2}\|\mX_i^k - \vu_i^k\vv_i^\T\|_F^2  + \frac{\lambda}{2}\|\vu_i^k - \vv_i\|_F^2 = \max\left(\frac{( \mX_{i}^k +\lambda \mId)\vu_i^{k}}{\|\vu_i^{k}\|_2^2+\lambda },0\right);
		$
		\STATE update residual as $\mX_i^{k} \leftarrow \mX_i^k - \vu_i^{k}(\vv_i^{k})^\T$.
		\ENDFOR
		\STATE$\mX_i^{k+1}  = \mX_i^{k}$, $k = k+1$.
		\ENDWHILE
	\end{algorithmic}
	{\bf Output:} factorization $(\mU_k, \mV_k)$.
\end{algorithm}

The SymHALS enjoys similar descend property and convergence guarantee to algorithm SymANLS as both of them are alternating-based algorithms. 

\begin{lem}\label{lem:sufficient decrease for HALS}
	Suppose the iterates sequence $\{(\mU_k,\mV_k)\}$ is generated by \Cref{alg:HALS}, then we have 
	\[
	f(\mU_k,\mV_k) - f(\mU_{k+1},\mV_{k+1}) \geq \frac{\lambda}{2}( \|\mU_{k+1}  - \mU_{k}\|_F^2 + \|\mV_{k+1}  - \mV_{k}\|_F^2).
	\]
\end{lem}
\begin{thm}[Sequence convergence of \Cref{alg:HALS}] For any $\lambda>0$, let $\{(\mU_k,\mV_k)\}$ be the sequence generated by \Cref{alg:HALS}.  Then 
	\[
	\lim\limits_{k\rightarrow \infty} (\mU_k,\mV_k) = (\mU^\star,\mV^\star)
	\]
	where $(\mU^\star,\mV^\star)$ is a critical point of \eqref{eq:SNMF by reg}. Furthermore the convergence rate is at least sublinear. 
\label{thm:HALS}\end{thm}

The proof of \Cref{lem:sufficient decrease for HALS} and \Cref{thm:HALS} follows similar arguments used for  \Cref{lem:sufficient decrease} and \Cref{thm:ANLS}. See the discussion in \Cref{sec:prf HALS}.

\begin{cor}[Convergence of \Cref{alg:HALS} to a critical point of \eqref{eq:SNMF}]
	Suppose it is initialized  with $\mV_0=\mU_0$. Choose 
	\[
	\lambda > \frac{1}{2} \left( \|\mX\|_2 + \left\|\mX-\mU_0\mU_0^\T\right\|_F - \sigma_n(\mX)  \right).
	\]	
	Let $\{(\mU_k,\mV_k)\}$ be the sequence generated by \Cref{alg:HALS}.  Then $\{(\mU_k,\mV_k)\}$ is convergent and converges to $(\mU^\star,\mV^\star)$ with $
	\mU^\star = \mV^\star $
	and $\mU^\star$	being a critical point of \eqref{eq:SNMF}. Furthermore, the convergence rate is at least sublinear.
	\label{cor:HALS}\end{cor}
\begin{proof}[Proof of \Cref{cor:HALS}]
	This corollary is a direct consequence of \Cref{thm:critical point U = V}, \Cref{lem:sufficient decrease for HALS} and \Cref{thm:HALS}. 
\end{proof}

\paragraph{Remark.} Similar to \Cref{cor:ANLS}, \Cref{cor:HALS} has no assumption on the boundedness of  iterates $(\mU_k,\mV_k)$ and it establishes convergence guarantee for SymHALS without the aid from a proximal term. As a contrast, to establish the subsequence convergence for classical HALS solving nonsymmetric NMF \cite{cichocki2009fast,gillis2012accelerated} (i.e., setting $\lambda = 0$ in SymHALS), one needs the assumption that every column of $(\mU_k,\mV_k)$ is not zero through all iterations. Though such assumption can be satisfied by using additional constraints, it actually solves a slightly different problem than the original nonsymmetric NMF \eqref{eq:NMF}. On the other hand,  SymHALS overcomes this issue and admits sequence convergence because of the additional regularizer in \eqref{eq:SNMF by reg}.

\section{Numerical Experiments}
\label{sec:experiments}
In this section, we conduct experiments on both synthetic data and real data to illustrate the performance of our proposed algorithms and compare it to other state-of-the-art ones, in terms of both  convergence property and image clustering performance.

For comparison convenience, we define 
\[
E^k = \frac{\|\mX-\mU^k(\mU^k)^\T\|_F^2}{\|\mX\|_F^2}
\]
as the normalized fitting error at $k$-th iteration.

 Besides SymANLS and SymHALS,  we also apply the  greedy coordinate descent (GCD) algorithm in \cite{hsieh2011fast} (which is designed for tackling nonsymmetric NMF) to solve the reformulated problem \eqref{eq:SNMF by reg} and denote the corrosponding algorithm as \textbf{SymGCD}.  SymGCD is expected to have similar sequence convergence guarantee as SymANLS and SymHALS. We list the algorithms to compare: 1) \textbf{ADMM} in \cite{lu2017nonconvex}, where there is a regularization item in their augmented Lagrangian and  we tune a good one for comparison;  2) \textbf{SymNewton} \cite{kuang2015symnmf} which is a Newton-like algorithm by with a the Hessian matrix in Newton's method for computation efficiency; and 3) \textbf{PGD} in \cite{lin2007projected}.  The algorithm in \cite{vandaele2016efficient} is inefficient for large scale data, since they apply an alternating minimization over each coordinate which entails many loops for large scale $\mU$. 
 
\subsection{Convergence verification}
We randomly generate a matrix $\mU\in\R^{50\times 5}(n=50,r=5)$ with each entry independently following a standard Gaussian distribution. To enforce nonnegativity, we then take absolute value on each entry of $\mU$ to get $\mU^\star$.  Data matrix $\mX$ is constructed as $\mU^\star (\mU^\star)^\T$ which is nonnegative and PSD.  We initialize all the algorithms with same $\mU^0$ and $\mV^0$, whose entries are \emph{i.i.d.} uniformly distributed  between 0 to 1.

\begin{figure}
	\centering
	\includegraphics[width=0.5\textwidth]{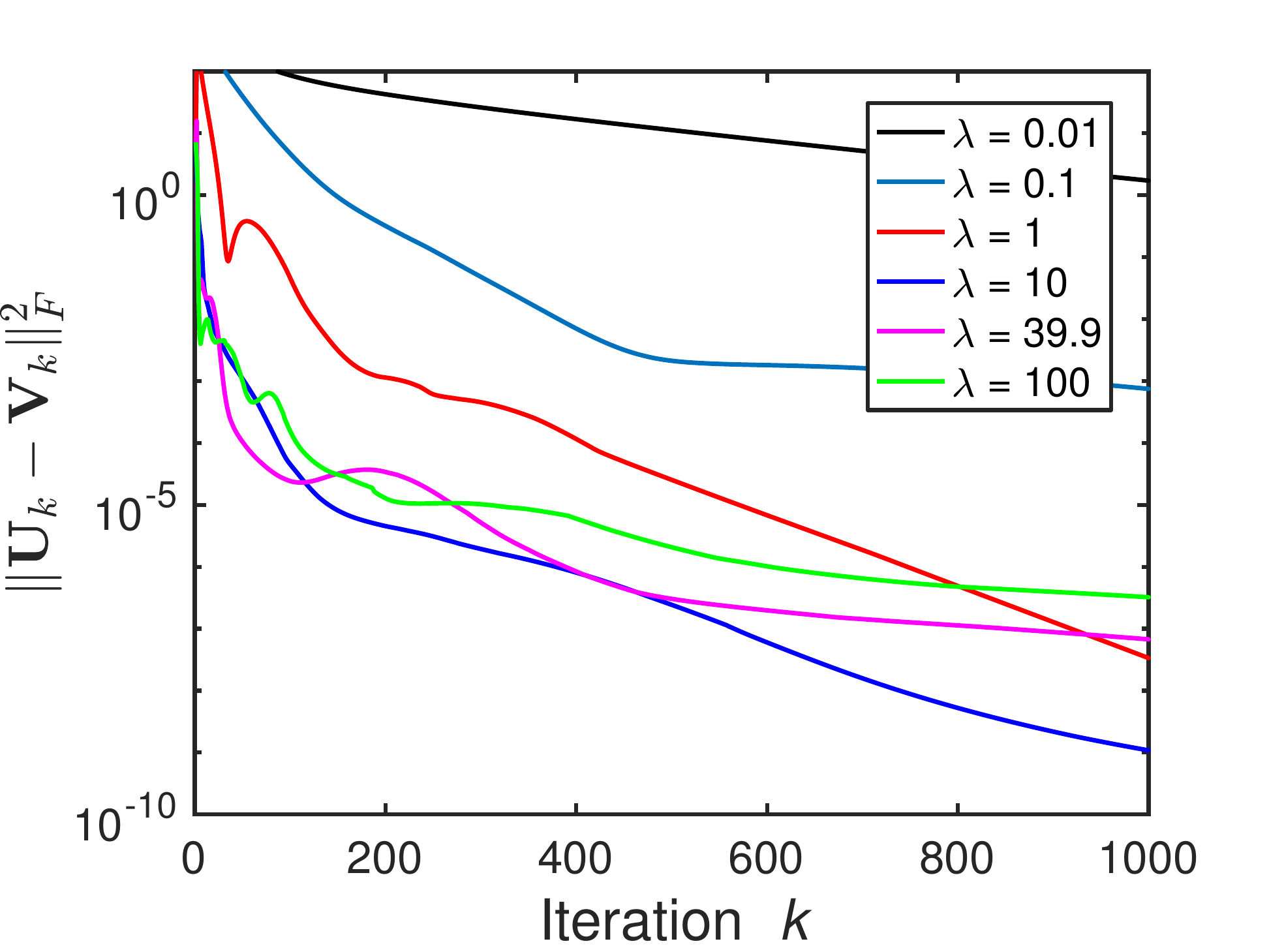}
	\caption{\footnotesize SymHALS with different $\lambda$. Here $n = 50, r = 5$.}\label{fig:lambda}
\end{figure}

To study the effect of the parameter $\lambda$ in \eqref{eq:SNMF by reg}, we show the value $\|\mU_{k}-\mV_{k}\|_F^2$ versus iteration for different choices of $\lambda$ by SymHALS in \Cref{fig:lambda}.  While for this experimental setting the lower bound of $\lambda$ provided in \Cref{thm:critical point U = V} is 39.9, we  observe that $\|\mU_{k}-\mV_{k}\|_F^2$ still converges to 0 with much smaller $\lambda$. This suggests that the \emph{sufficient condition} on the choice of $\lambda$ in \Cref{thm:critical point U = V} is stronger than necessary, leaving room for future improvements. Particularly, we suspect that SymHALS converges to a critical point $(\mU^\star, \mV^\star)$ with $\mU^\star = \mV^\star$ (i.e. a critical point of symmetric NMF) for any $\lambda >0$; we leave this line of theoretical justification as our future work. On the other hand, we note that although SymHALS finds a critical point of symmetric NMF for most of the $\lambda$, the convergence speed varies for different $\lambda$. For example, we observe that either a very large or small $\lambda$ yields a slow convergence speed. In the sequel, we tune the best parameter $\lambda$ for each experiment. 

We also  test on real world dataset CBCL~\footnote{http://cbcl.mit.edu/software-datasets/FaceData2.html}, where there are 2429 face image data with dimension $19\times 19$. We construct the similarity matrix $\mX$ following  ~\cite[section 7.1, step 1 to step 3]{kuang2015symnmf}. The convergence results on synthetic data and real world data are shown in \Cref{fig:synthetic} (a1)-(a2) and \Cref{fig:synthetic} (b1)-(b2), respectively. We observe that the SymANLS, SymHALS, and SymGCD 1) converge faster; 2) empirically have a  linear convergence rate  in terms of $E^k$.

\begin{figure}[!htb]  
	\centering
	\begin{minipage}{0.48\linewidth}
		\includegraphics[width=8cm]{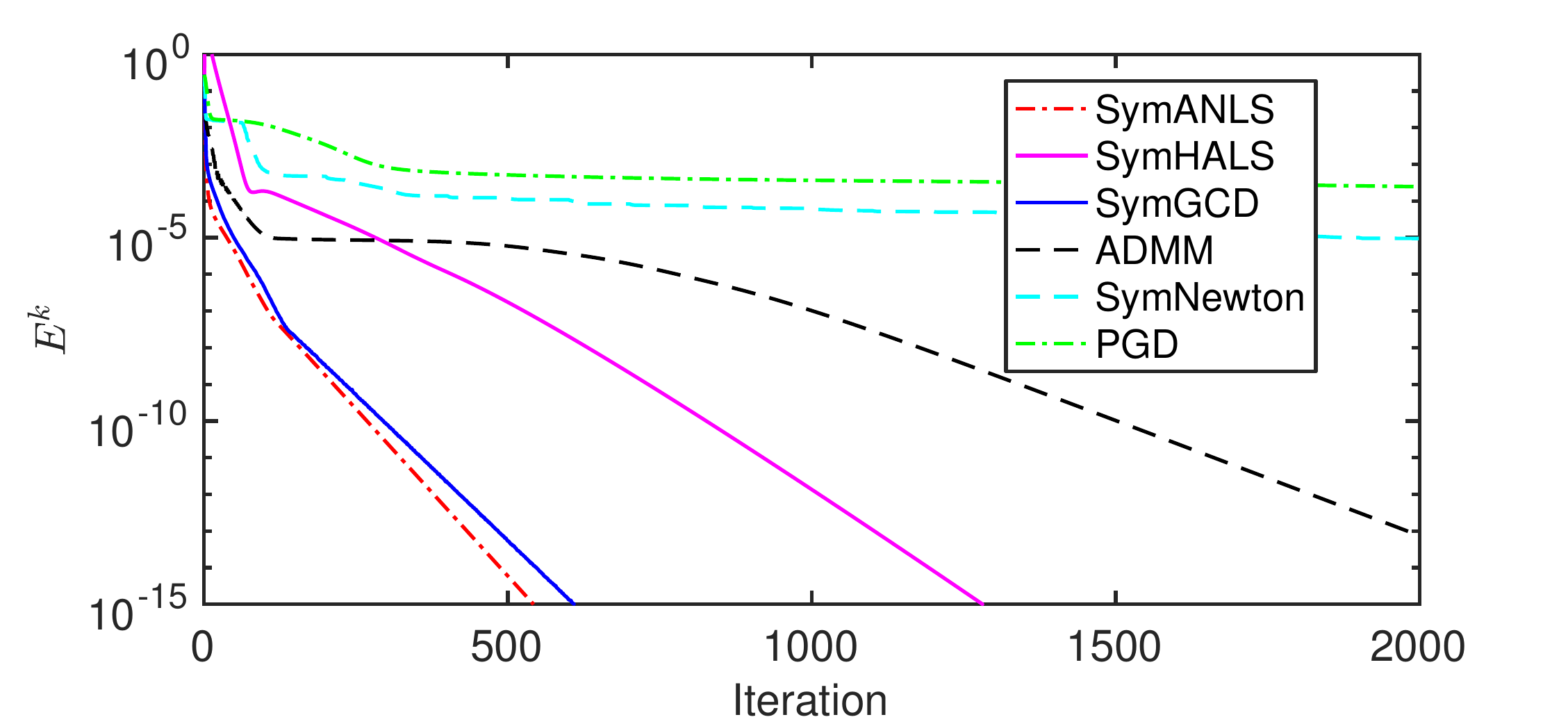}
		\centerline{(a1)}
	\end{minipage}
	\hfill
	\begin{minipage}{0.48\linewidth}
		\includegraphics[width=8cm]{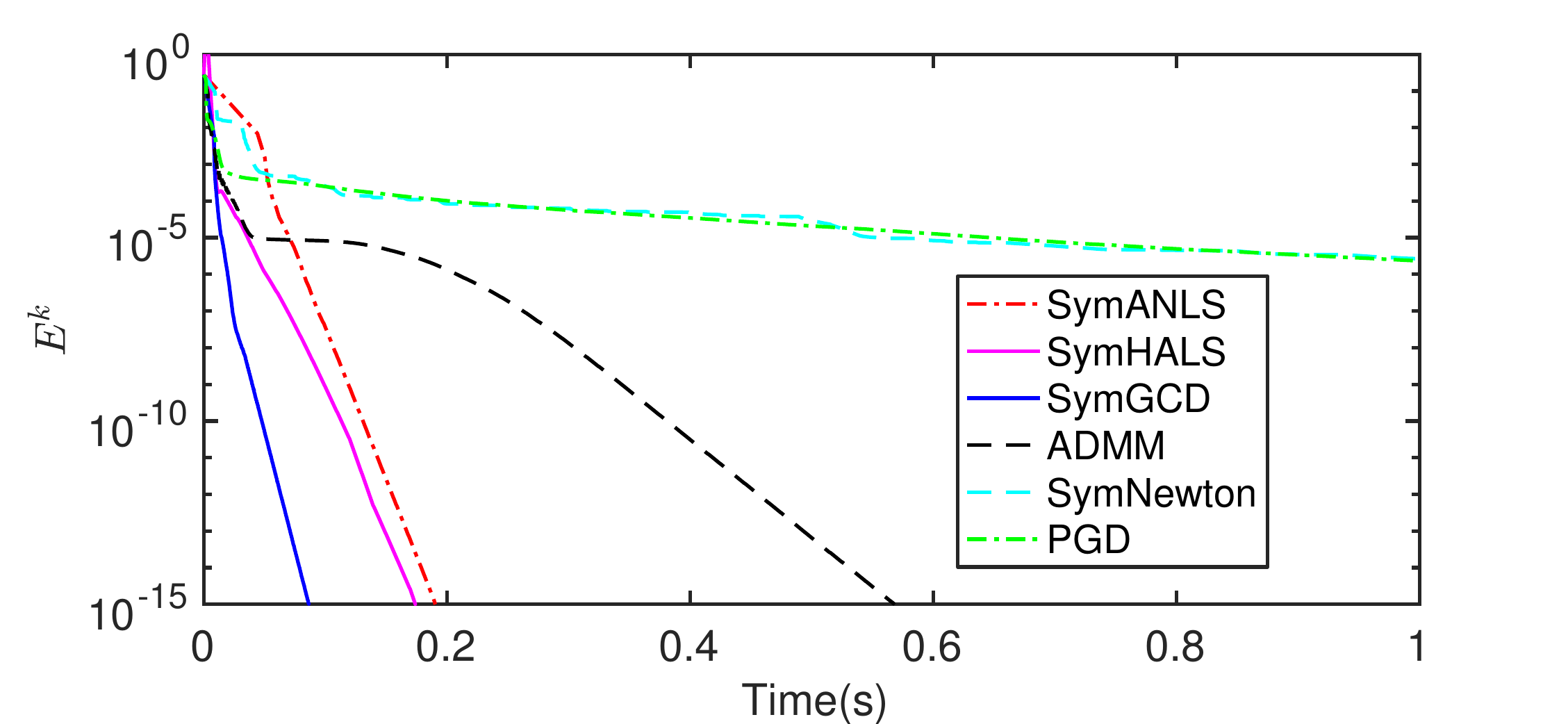}
		\centerline{(a2)}
	\end{minipage}
	\vfill	
	\begin{minipage}{0.48\linewidth}
		\includegraphics[width=8cm]{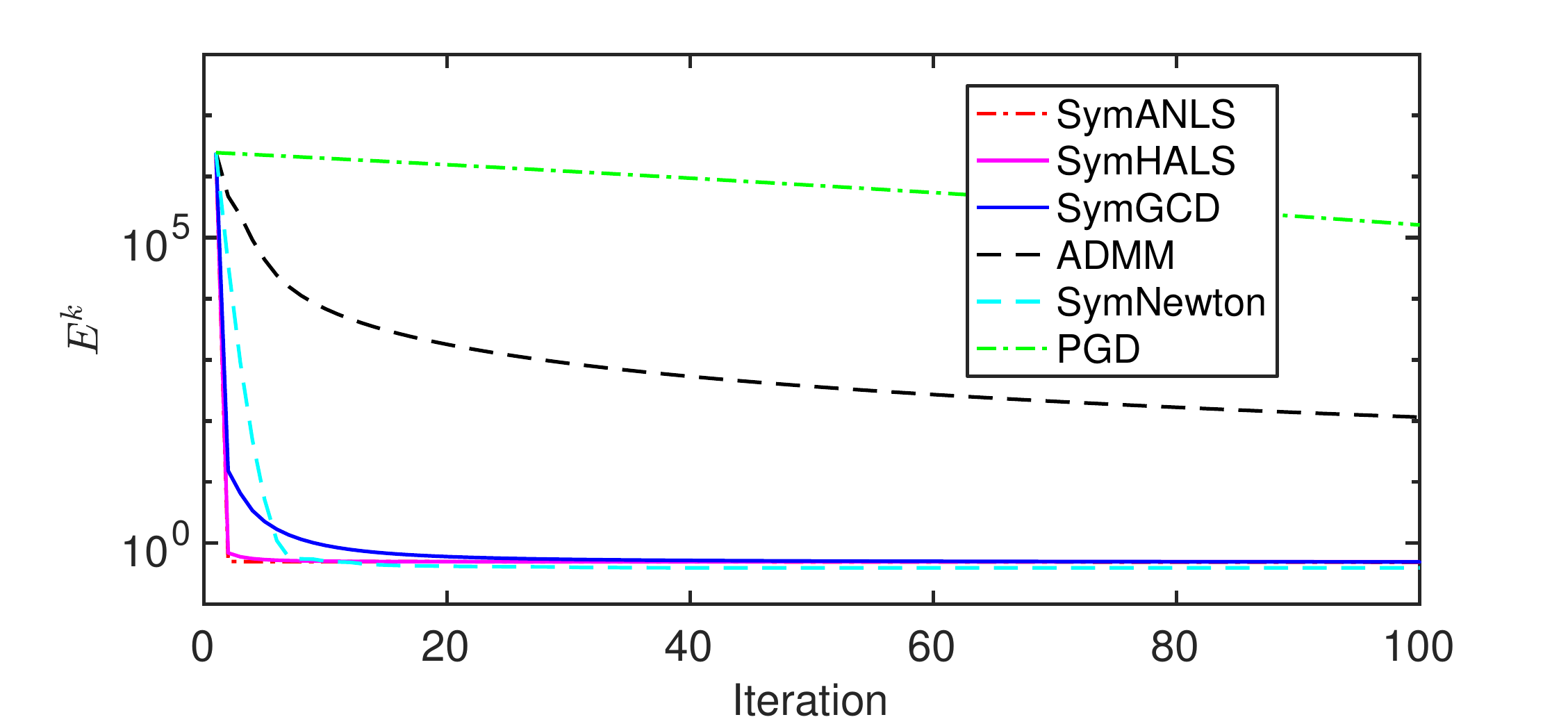}
		\centerline{(b1)}
	\end{minipage}
	\hfill
	\begin{minipage}{0.48\linewidth}
		\includegraphics[width=8cm]{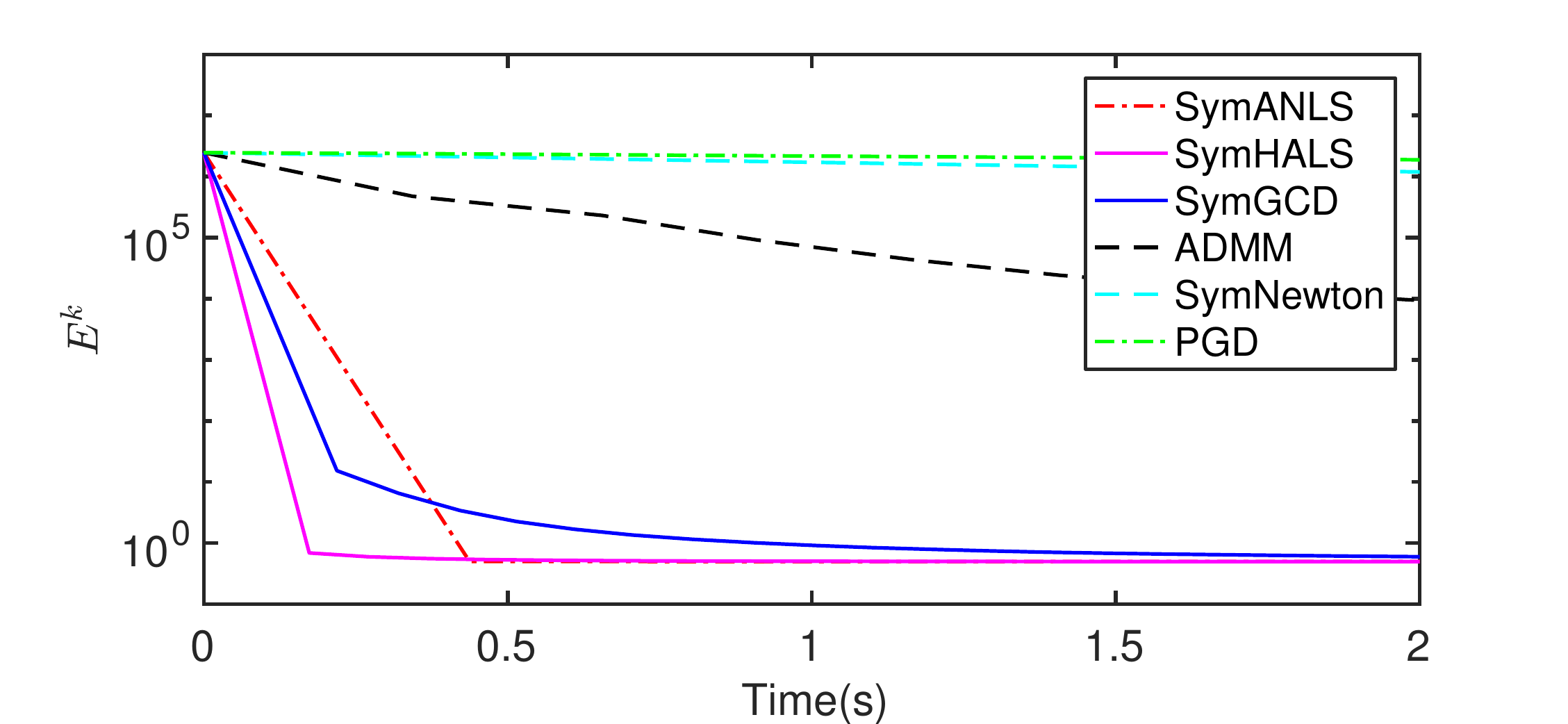}
		\centerline{(b2)}
	\end{minipage}
	\caption{Synthetic data where $n= 50, r = 5$: (a1)-(a2) fitting error versus iteration and running time. Real image dataset CBCL where $n = 2429, r = 49$: (b1)-(b2) fitting error versus iteration and running time.} \label{fig:synthetic}
\end{figure}

\subsection{Image clustering}
Symmetric NMF can be used for graph clustering~\cite{kuang2015symnmf,ding2005equivalence} where each element $\mX_{ij}$ denotes the similarity between data $i$ and $j$. 
In this subsection, we apply different symmetric NMF algorithms for graph clustering on image datasets and compare the clustering accuracy~\cite{xu2003document}.

We put all images to be clustered in a data matrix $\mM$, where each row is a vectorized image.  We construct similarity matrix following the procedures in ~\cite[section 7.1, step 1 to step 3]{kuang2015symnmf}, and utilize self-tuning method to construct the similarity matrix $\mX$. 
Upon deriving $\widetilde \mU$ from symmetric NMF  $\mX \approx \widetilde \mU\widetilde \mU^\T$, the label of the $i$-th image can be obtained by: 
\begin{equation}\label{eq:assign_label}
l({\mM}_i) = \argmax_{j}{\widetilde \mU}_{(ij)}.
\end{equation}
We  conduct the experiments on four image datasets:

\textbf{ORL}: 400 facial images from 40 different persons with each one has 10 images from different angles and emotions~\footnote{http://www.cl.cam.ac.uk/research/dtg/attarchive/facedatabase.html}.

\textbf{COIL-20}: 1440 images from 20 objects~ \footnote{http://www.cs.columbia.edu/CAVE/software/softlib/coil-20.php}.

\textbf{TDT2}: 10,212 news articles from 30 categories~\footnote{https://www.ldc.upenn.edu/collaborations/past-projects}. We extract the first 3147 data for experiments (containing only 2 categories). 

\textbf{MNIST}:  classical handwritten digits dataset~\footnote{http://yann.lecun.com/exdb/mnist/}, where 60,000 are for training (denoted as MNIST$_{train}$), and 10,000 for testing (denoted as MNIST$_{test}$). we test on the first 3147 data from MNIST$_{train}$ (contains 10 digits)  and 3147 from MNIST$_{test}$ (contains only 3 digits) .

In \Cref{fig:real date} (a1) and \Cref{fig:real date}(a2), we display  the  clustering accuracy on dataset \textbf{ORL} with respect to iterations and time (only show first 10  seconds), respectively. Similar results for dataset \textbf{COIL-20} are plotted in  \Cref{fig:real date} (b1)-(b2).  We observe that in terms of iteration number,  SymNewton has comparable performance to the three alternating methods for \eqref{eq:SNMF by reg} (i.e., SymANLS, SymHALS, and SymGCD), but the latter outperform the former in terms of running time. Such superiority becomes more apparent when the size of the dataset increases.   We note that the  performance of ADMM will increase as iterations goes and after almost 3500 iterations on ORL dataset  it reaches a comparable result to  other algorithms. Moreover, it requires more iterations for larger dataset. This observation makes ADMM not  practical  for image clustering. We run ADMM 5000 iterations on ORL dataset; see \Cref{fig:ADMM ORL}. These results as well as the experimental results shown in the last subsection demonstrate $(i)$ the power of transfering the symmetric NMF \eqref{eq:SNMF} to a nonsymmetric one \eqref{eq:SNMF by reg}; and $(ii)$ the efficieny of alternating-type algorithms for sovling \eqref{eq:SNMF by reg} by exploiting the splitting property within the optimization variables in \eqref{eq:SNMF by reg}.

\begin{figure}[!htb]
	\centering
	
	\begin{minipage}{0.48\linewidth}
		\includegraphics[width=8cm]{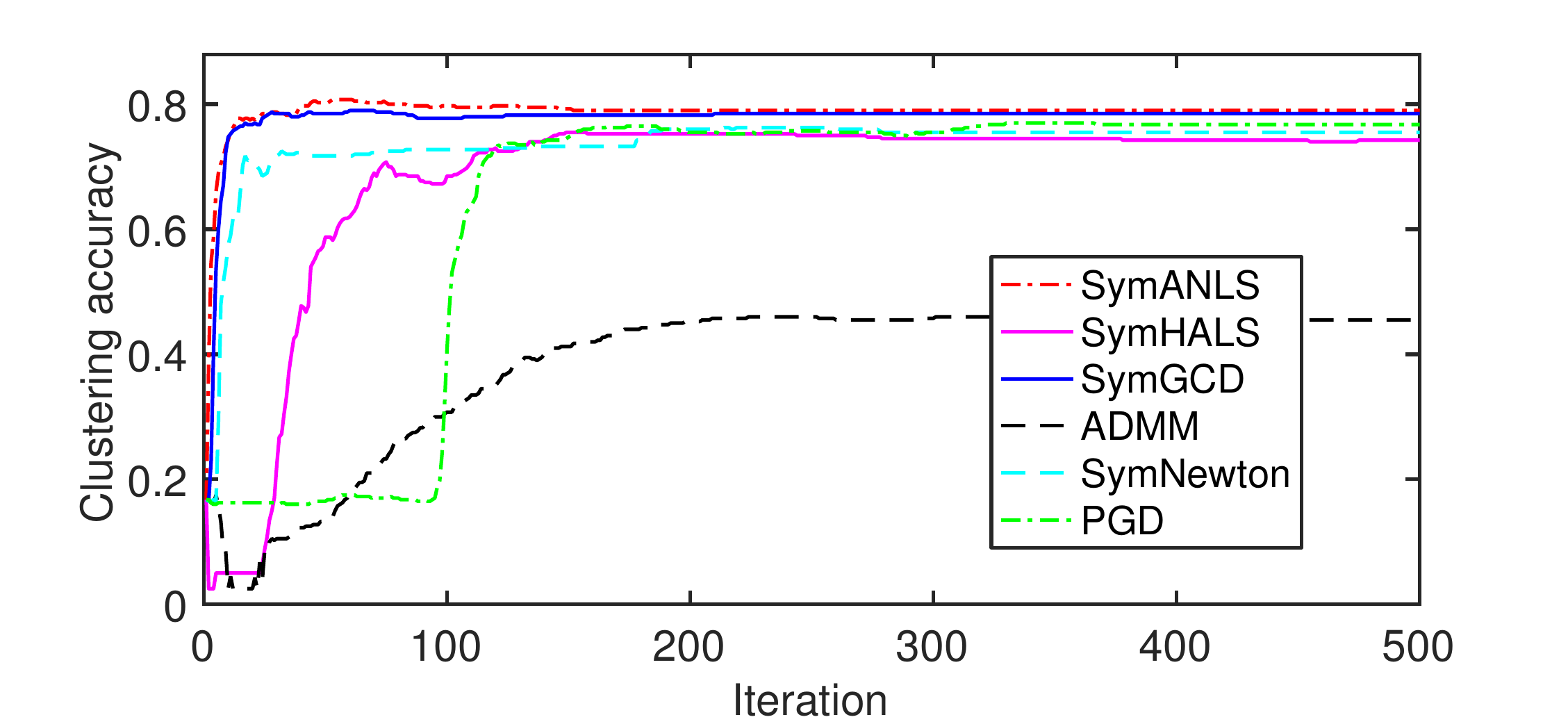}\\
		\centering{(a1)}
	\end{minipage}
	\hfill
	\begin{minipage}{0.48\linewidth}
		\includegraphics[width=8cm]{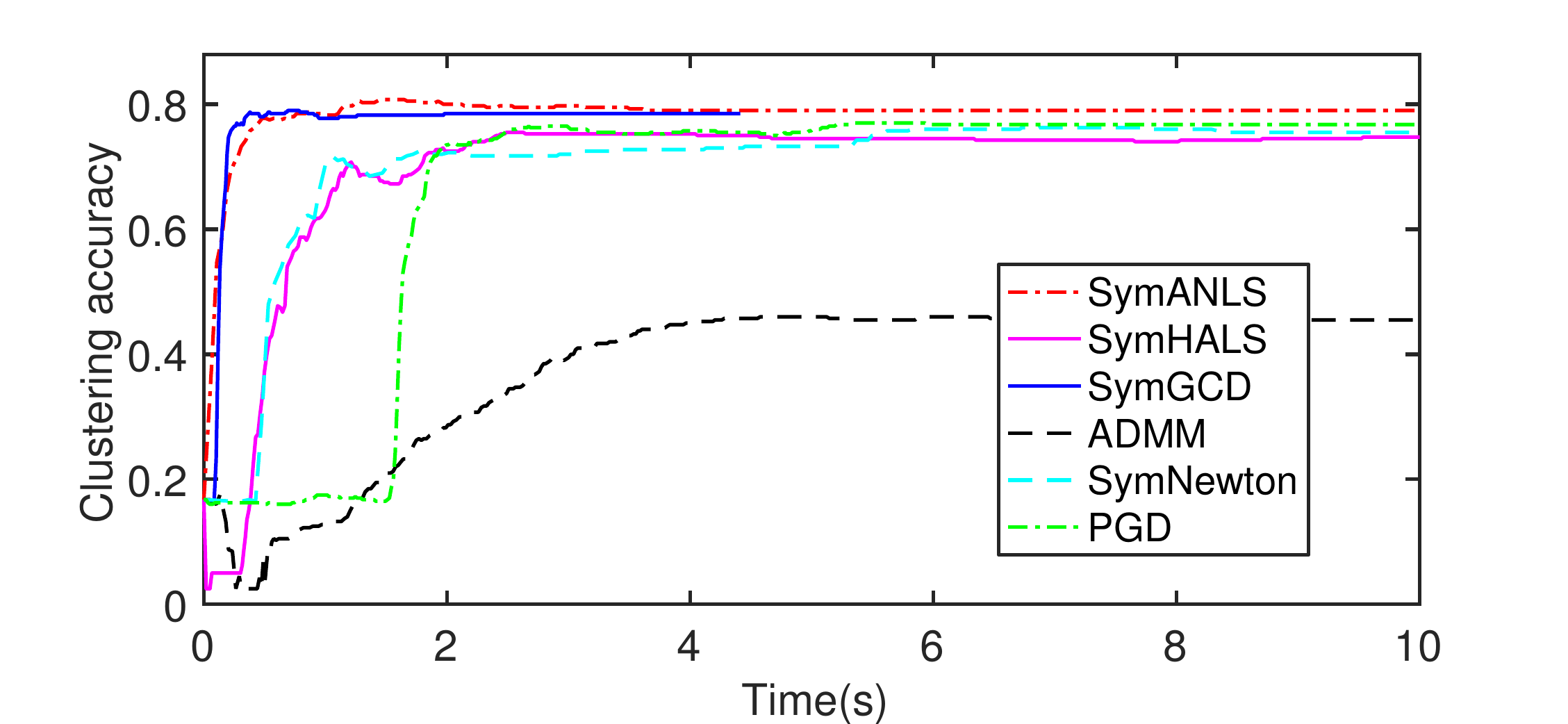}\\
		\centering{(a2)}
	\end{minipage}
	\vfill
	\begin{minipage}{0.48\linewidth}
		\includegraphics[width=8cm]{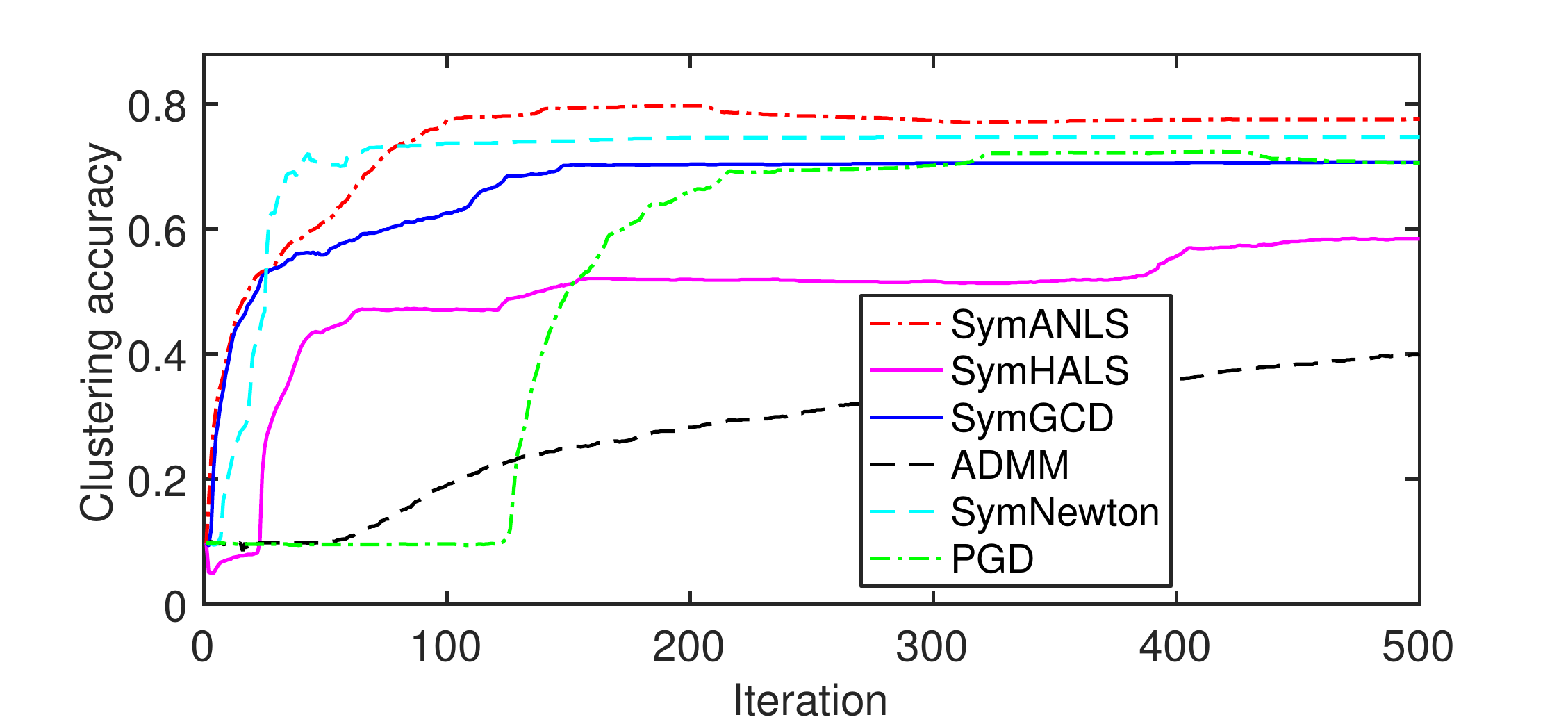}\\
		\centering{(b1)}
	\end{minipage}
	\hfill	
	\begin{minipage}{0.48\linewidth}
		\includegraphics[width=8cm]{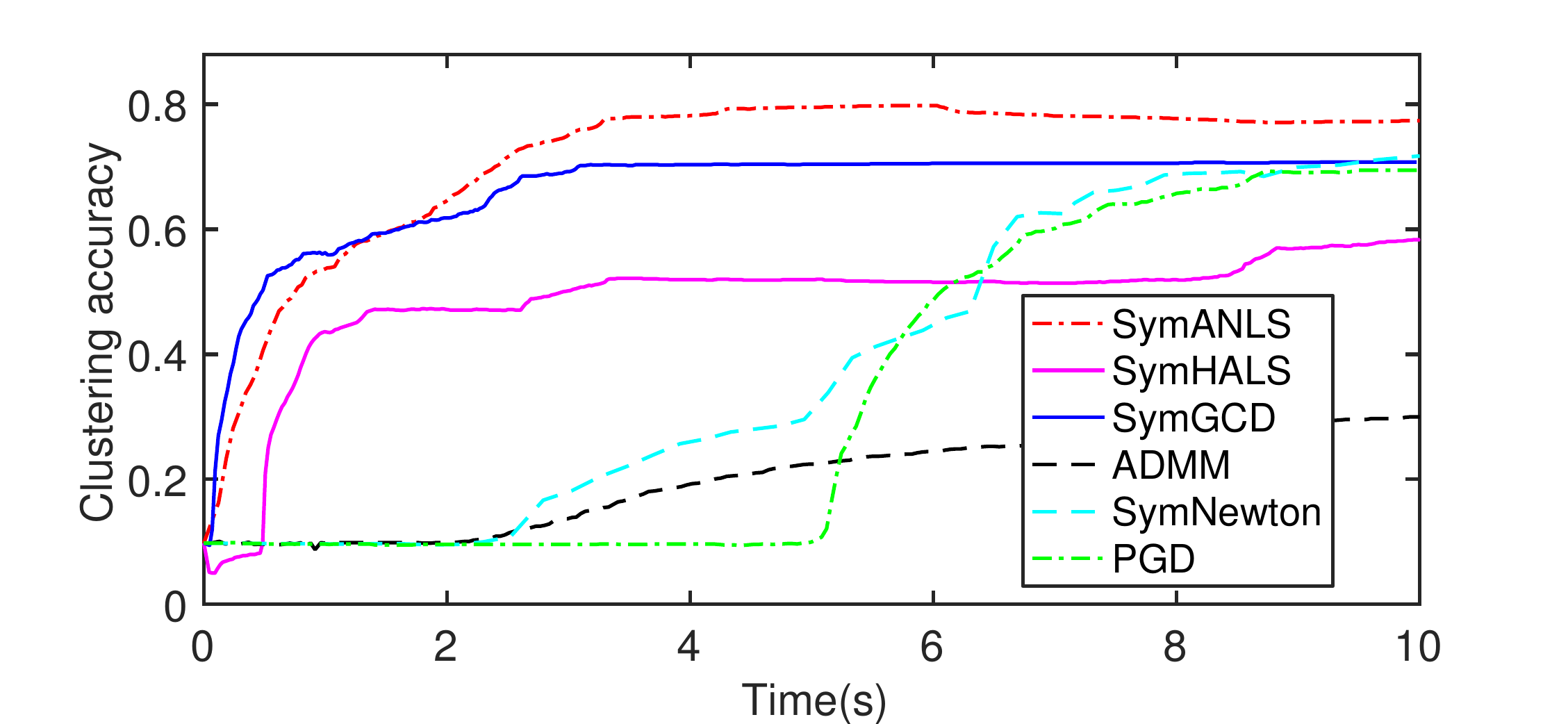}\\
		\centering{(b2)}
	\end{minipage}
	\caption{Real dataset:  (a1) and (a2) Image clustering quality on ORL dataset, $n= 400, r = 40$; (b1) and (b2) Image clustering quality on COIL-20 dataset, $n= 1440, r = 20$. }  \label{fig:real date}
\end{figure}

\begin{figure}[!htb] 
	\centering
	\includegraphics[width=8cm]{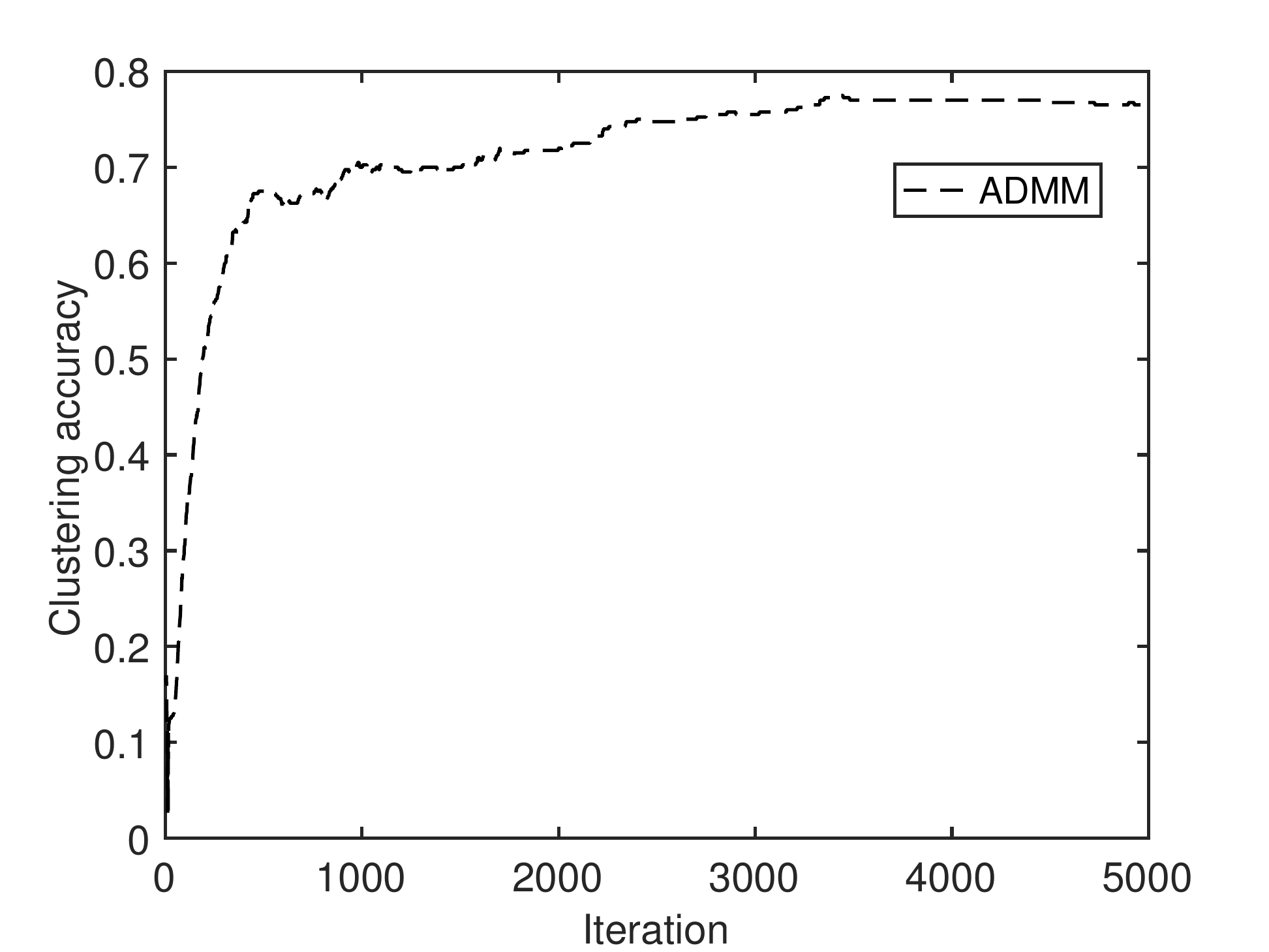}
	\caption{Image clustering quality of ADMM on ORL dataset, here $n = 400, r = 40$. ADMM needs roughly 3500 iterations to reach its maximum clustering rate!}
	\label{fig:ADMM ORL}
\end{figure}

Table~\ref{table:image clustering} shows the clustering accuracies of different algorithms on different datasets, where  we run enough iterations for ADMM so that it obtains its best result. We observe from Table~\ref{table:image clustering} that SymANLS, SymHALS, and SymGCD perform better than or have comparable performance to others for most of the cases. 

%

\begin{table*}[htb!]\caption{Summary of image clustering accuracy of different algorithms on five image datasets}\label{table:image clustering}
	\begin{center}
		\begin{tabular}{c|c|c|c|c|c}
			\hline  
			&ORL&COIL-20&MNIST$_{train}$&TDT2&MNIST$_{test}$\\
			\hline  
			SymANLS&\textbf{0.8075}&\textbf{0.7979}&0.6477&0.9800&0.8589\\
			\hline 
			SymHALS&0.7550&0.5854&\textbf{0.6657}&\bf{0.9806}&{0.8608}\\
			\hline 
			SymGCD&0.7900&0.7076&{0.6293}&{0.9803}&\bf{0.9882}\\
			\hline 
			ADMM&0.7650&0.6903&0.5803&0.9800&{0.8713}\\
			\hline 
			SymNewton&0.7625&0.7472&0.5990&0.9793&0.8589\\
			\hline 
			PGD&0.7700&0.7243&0.6475&0.9800&0.8710\\
			\hline 
		\end{tabular}
	\end{center}
\end{table*}

\section{Proofs for \Cref{sec:fast algorithms}}

Before going to the main proof, we  first introduce some supporting materials.
\subsection{Definitions and basic ingredients}
Since problem \eqref{eq:SNMF by reg} is nonsmooth and nonconvex, we use tools from generalized differentiation to  characterize its optimality. 
\begin{defi}\cite{bolte2014proximal,rockafellar2015convex} \label{def:subdifferentials}Let $h:\R^d\rightarrow (-\infty,\infty]$ be a proper and lower semi-continuous function
	\begin{enumerate}[$(i)$]
		\item the  effective domain  is defined as
		\[
		\domain h:=\left\{\vu\in\R^d:h(\vu)<\infty\right\}.
		\]
		\item
		The (Fr\'{e}chet) subdifferential $\partial h$ of $h$ at $\vu$ is defined by
		\[
		\partial h(\vu) = \left\{\vz:\liminf\limits_{\vv\rightarrow \vu,\vv\neq \vu}\frac{h(\vv) - h(\vu) - \langle \vz, \vv - \vu\rangle}{\|\vu - \vv\|}\geq 0\right\}
		\]
		for any $\vu\in \domain h$ and $\partial h(\vu) = \emptyset$ if $\vu\notin \domain h$.
	\end{enumerate}	
\end{defi}
When $h(\vu)$ is differentiable at $\vu$, or is convex, then the (Fr\'{e}chet) subdifferential reduces to $\nabla h(\vu)$ or the  convex subdifferential. Throughout this section, we will simply say $\partial h(\vu)$ subdifferential. 
A necessary condition for optimality is $0 \in\partial h(\overline \vu)$, and such a point is called critical point of $h(\vu)$.

Note that for a nonnegative constraint which is convex, it is subdifferentialble everywhere in its effective domain, including the relative boundary. 
For the objective function $f$ in \eqref{eq:SNMF by reg}, the subdiferential is simply given by: gradient of its smooth part $+$ the convex subdifferential of its nonnegative constraints, where the `$+$' represents Minkowsiki summation of sets. Denote by $g(\mU,\mV) = \frac{1}{2}\|\mX - \mU\mV^\T\|_F^2 + \frac{\lambda}{2}\|\mU - \mV\|_F^2$.
\begin{lem}The subdifferential of $f$ is given as follows:
\e\begin{split}
&\partial_{\mU} f(\mU,\mV) = \nabla_{\mU} g(\mU,\mV)+\partial \delta_{+}(\mU), \\
& \partial_{\mV} f(\mU,\mV) =  \nabla_{\mV} g(\mU,\mV)+\partial \delta_{+}(\mV),	
\end{split}\label{eq:subdiff for f}\ee
where 
\e\begin{split}
\partial \delta_{+}(\mU) &= \left\{\mS\in\R^{n\times r}: \delta_+(\mU') - \delta_+(\mU)\geq \langle \mS,\mU' - \mU \rangle, \ \forall \mU'\in\R^{n\times r}\right\}\\
& = \left\{\mS\in\R^{n\times r}: \mS \leq \vzero, \  \mS \odot \mU = \vzero  \right\}
\end{split}
\nonumber\ee
\label{lem:subdiff for f}\end{lem}

Another widely used definition for a stationary point is via the Karush-Kuhn-Tucker (KKT) condition.  The KKT condition for problem \eqref{eq:SNMF by reg} is stated as follows.
\e 
\begin{split}
	& \mU\geq 0, \ \mV\geq 0\\
	&\nabla_{\mU} g(\mU,\mV) \geq 0, \ \nabla_{\mV} g(\mU,\mV) \geq 0\\
	&\mU \odot \nabla_{\mU} g(\mU,\mV)=0, \  \mV \odot  \nabla_{\mV} g(\mU,\mV) = 0.
\end{split}
\nonumber\ee
We say $(\mU^\star,\mV^\star)$ a KKT point if it satsifies the above equations. The following result establises the equivalence between a KKT point and a critical point defined from the notion of subdifferentional.
\begin{prop}\label{prop:critical and KKT}
	For problem \eqref{eq:SNMF by reg}, $(\mU^\star,\mV^\star)$ is a critical point iff it is a KKT point.
\end{prop}
\begin{proof}[Proof of \Cref{prop:critical and KKT}]
	Suppose $(\mU^\star,\mV^\star)$ is a critical point of  \eqref{eq:NMF}:
	\[
	\vzero \in \partial f(\mU^\star,\mV^\star) = (\nabla_{\mU} g(\mU^\star,\mV^\star)+\partial \delta_{+}(\mU^\star),  \nabla_{\mV} g(\mU^\star,\mV^\star)+\partial \delta_{+}(\mV^\star)),
	\]
which implies that $\mU^\star\geq 0, \mV^\star\geq 0$ since otherwise from the definition of convex subdifferential, both $\partial \delta_{+}(\mU^\star)$ and $\partial \delta_{+}(\mV^\star)$ are empty. We define $\mS^\star\in \partial \delta_{+}(\mU^\star)$ and $\mD^\star \in \partial \delta_{+}(\mV^\star)$ such that 
	\e \label{eq:element-wise critical point equation}
	\nabla_{\mU} g(\mU^\star,\mV^\star) + \mS^\star = 0, \ \ \nabla_{\mV} g(\mU^\star,\mV^\star) + \mD^\star = 0.
	\ee
It follows from the definition of convex subdifferential and  separability of the nonnegative constraint indicator function that 
	\e \label{eq:element-wise subdifferential}
	\begin{split}
		&\delta_{+}(\mU_{ij}) \geq \delta_{+}(\mU^\star_{ij})  + \langle\mS^\star_{ij}\ve_i\ve_j^\T, (\mU_{ij} -   \mU^\star_{ij})\ve_i\ve_j^\T\rangle, \ \ \forall i,j, \ \ \forall \ \mU\in \domain f,\\
		&\delta_{+}(\mV_{ij}) \geq \delta_{+}(\mV^\star_{ij})  + \langle\mD^\star_{ij}\ve_i\ve_j^\T, (\mV_{ij} -   \mV^\star_{ij})\ve_i\ve_j^\T\rangle, \ \ \forall i,j, \ \  \forall \ \mV\in \domain f.
	\end{split}
	\ee
Constructing $\mU = 2\mU^\star$ and $\mV = 2\mV^\star$  gives 
	\[
	\mS^\star_{ij} \mU^\star_{ij} \leq 0, \ \ \mD^\star_{ij} \mV^\star_{ij} \leq 0, \ \ \forall i,j.
	\]
Similarly, plugging $\mU = 0$ and $\mV = 0$  gives
	\[
	\mS^\star_{ij} \mU^\star_{ij} \geq 0, \ \ \mD^\star_{ij} \mV^\star_{ij} \geq 0, \ \ \forall i,j.
	\]
Hence, 
	\[
	\mS^\star_{ij} \mU^\star_{ij} = 0, \ \ \mD^\star_{ij}\mV^\star_{ij} = 0, \ \ \forall i,j.
	\]
which together with \eqref{eq:element-wise critical point equation} implies 
	\[
	[\nabla_{\mU} g(\mU^\star,\mV^\star)]_{ij} \mU^\star_{ij} = 0, \ \ [\nabla_{\mV} g(\mU^\star,\mV^\star)]_{ij} \mV^\star_{ij} = 0, \ \ \forall i,j.
	\]
Thus $(\mU^\star,\mV^\star)$ satisfies complementary slackness equation in the KKT condition. 
	
When $\mU^\star_{ij}>0$, from the above complementary $[\nabla_{\mU} g(\mU^\star,\mV^\star)]_{ij} = 0,[\nabla_{\mV} g(\mU^\star,\mV^\star)]_{ij} \geq 0$, the second equation in the KKT condition is satisfied for these $(i,j)$-th entry. For those $(i,j)$-th entries such that $\mU^\star_{ij}=0$, plugging any $\mU_{ij}>0$ and $\mV_{ij}>0$  into \eqref{eq:element-wise subdifferential} provides $\mS^\star_{ij}\leq 0$, which together with \eqref{eq:element-wise critical point equation} gives $[\nabla_{\mU} g(\mU^\star,\mV^\star)]_{ij} \geq 0, [\nabla_{\mV} g(\mU^\star,\mV^\star)]_{ij} \geq 0$. Hence the second equation in the KKT condition also holds true for all $i,j$. Therefore, $(\mU^\star,\mV^\star)$ satisfies the KKT condition. 
	
	On the contrary, suppose $(\mU^\star,\mV^\star)$ is a KKT point of  \eqref{eq:NMF}. We have 
	\e 
	\begin{split}
		& \mU^\star\geq 0, \ \mV^\star\geq 0\\
		&\mLambda^\star_{1} \geq 0, \ \mLambda^\star_{2}\geq 0, \\
		&\nabla_{\mU} g(\mU^\star,\mV^\star) - \mLambda^\star_{1} = 0, \ \nabla_{\mV} g(\mU,\mV) - \mLambda^\star_{2} = 0\\
		&\mU \odot \mLambda^\star_{1}=0, \  \mV \odot \mLambda^\star_{2} = 0.
	\end{split}
	\ee
Reversing the above arguments leads to $\mLambda^\star_{1} \in \partial \delta_{+}(\mU^\star)$ and $ \mLambda^\star_{2} \in \partial \delta_{+}(\mV^\star)$, and we conclude 
	\[
	0 \in \partial f(\mU^\star,\mV^\star) = (\nabla_{\mU} g(\mU^\star,\mV^\star)+\partial \delta_{+}(\mU^\star),  \nabla_{\mV} g(\mU^\star,\mV^\star)+\partial \delta_{+}(\mV^\star)),
	\]
which implies that $(\mU^\star,\mV^\star)$ is a critical point of \eqref{eq:NMF}.
\end{proof}

The following property states the geometry of objective function (including its constraints) around its critical points, which plays a key role in our sequel analysis. 
\begin{defi}\cite{bolte2007lojasiewicz,attouch2009convergence}\label{def:KL}
	We say a proper semi-continuous function $h(\vu)$ satisfies Kurdyka-Lojasiewicz (KL) property, if $\overline{\vu}$ is a limiting critical point of $h(\vu)$, then there exist $\delta>0,~\theta\in[0,1),~C_1>0,~s.t.$
	\[
	\left|h(\vu) - h(\overline{\vu})\right|^{\theta} \leq C_1 \dist(0, \partial h(\vu)),~~\forall~\vu\in B(\overline{\vu}, \delta)
	\]
	
\end{defi}
The above KL property (also known as KL inequality) states the regularity of $h(\vu)$ around its critical point $\vu$.   A very large set of functions satisfy the KL inequality. For example, as stated in \cite[Theorem 5.1]{bolte2014proximal}.   a proper lower semi-continuous function has KL property once it has semi-algebraic property which is sufficiently general, including but never limited to any polynomials, any norm, quasi norm, $\ell_0$ norm, smooth manifold, etc. For more discussions and examples, see \cite{bolte2014proximal,attouch2013convergence}.

\subsection{Proof of \Cref{lem:sufficient decrease}}
\label{sec:prf lem sufficient decrease}

We first show that the smooth part of the objective function in \eqref{eq:SNMF by reg} is $C^1$ smooth on any bounded subset.
\begin{lem} 
	The function $g(\mU,\mV) = \frac{1}{2}\|\mX - \mU\mV^\T\|_F^2 + \frac{\lambda}{2}\|\mU - \mV\|_F^2$
	has  Lipschitz continuous gradient with the Lipschitz constant as $2B+\lambda+\|\mX\|_F$
	in any bounded $\ell_2$-norm ball $\{(\mU,\mV): \|\mU\|_F^2+\|\mV\|_F^2\leq B\}$ for any  $B>0$.
	\label{lem:lipchitz}
\end{lem} 
\begin{proof}
To ease the notation, we stack $\mU$ and $\mV$ into one variable $\mW:=(\mU,\mV)$. To obtain the Lipschitz constant, it is equivalent to bound the spectral norm of the quadrature form of the Hessian $[\nabla^2 g(\mW)](\mD,\mD)$ for any $\mD:=(\mD_U,\mD_V)$:
	\begin{align*}
	[\nabla^2 g(\mW)](\mD,\mD)
	=&
	\|\mU\mD_V^\T+\mD_U\mV^\T\|_F^2-2\lg\mX,\mD_U\mD_V^\T\rg+\frac{\lambda}{2}\|\mD_V-\mD_U\|_F^2\\
	&\leq2\|\mU\|_F^2\|\mD_V\|_F^2+2\|\mV\|_F^2\|\mD_U\|_F^2+\underbrace{\lambda\|\mD_U\|_F^2+\lambda\|\mD_V\|_F^2}_{=\lambda\|\mD\|_F^2}+2\|\mX\|_F\underbrace{\|\mD_U\mD_V^\T\|_F}_{\leq \|\mD\|_F^2/2}\\
	&\leq(2\|\mU\|_F^2+2\|\mV\|_F^2+\lambda+\|\mX\|_F)\|\mD\|_F^2\leq (2B+\lambda+\|\mX\|_F)\|\mD\|_F^2.
	\end{align*}
\end{proof}

 As each iterate $\mW_k = (\mU_k,\mV_k)$  lives in the $\ell_2$-norm ball with the radius $\sqrt{B_0}$ (see \eqref{eqn:bound}),  $g$ has Lipschitz continuous gradient with the Lipschitz constant being $2B_0+\lambda+\|\mX\|_F$ around each $\mW_k$. We now prove \Cref{lem:sufficient decrease}.
\begin{proof}[Proof of \Cref{lem:sufficient decrease}]
	Updating $\mV_{k +1  }$ amounts to solve
	\[
	\minimize_{\mV} f(\mU_{k},\mV) = g(\mU_{k},\mV) + \sigma_{+}(\mV).
	\]

As  the indicator function of nonnegative constraint $\sigma_{+}(\mV)$ is  convex subdifferentiable  for all $\mV$ in its effective domain including relative boundary, its  subdifferential is given as follows
	\[
	\partial \sigma_{+}(\mV) = \{ \mS\in \R^{n\times r}: \sigma_{+}(\mV) \geq \sigma_{+}(\widetilde\mV) + \langle \mS, \mV-\widetilde\mV\rangle,~~\forall~\mV \in \R^{n\times r} \}
	\]
which	is nonempty for all  $\mV\geq 0$. Utilizing the nonnegativity of  $\mV_{k}$ and  $\mV_{k +1}$  gives
	\e
	0\geq \langle \mS_{k+1}, \mV_{k}-\mV_{k +1}\rangle,~~\forall~\mS_{k+1} \in  \partial \delta_{+}(\mV_{k +1}).
\label{eq:S V negative}	\ee
	Since the update means $\mV_{k +1} = \argmin_{\mV}  g(\mU_{k},\mV) + \sigma_{+}(\mV)$,  it can be seen from the first order optimality ${\bf 0} \in \nabla_{\mV}g(\mU_{k},\mV_{k +1}) +  \partial \sigma_{+}(\mV_{k +1})$ that
	\[
	\nabla_{\mV}g(\mU_{k},\mV_{k +1})+ \mS_{k+1} = {\bf 0}.
	\]
Multiplying $\mV_{k}-\mV_{k +1}$ on both sides in the above equation provides
	\[
	\left\langle \nabla_{\mV}g(\mU_{k},\mV_{k +1}), \mV_{k}-\mV_{k +1}\right\rangle + \langle\mS_{k+1} , \mV_{k}-\mV_{k +1}\rangle  = {\bf 0},
	\]
which together with \eqref{eq:S V negative} gives
	\begin{equation}
	\left\langle \nabla_{\mV}g(\mU_{k},\mV_{k +1}), \mV_{k}-\mV_{k +1}\right\rangle    \geq 0. \label{eq:subdiff_property 2}
	\end{equation}
	
	Now utilizing the Taylor expansion,
	\begin{align*}
	g(\mU_k,\mV_k) &= g(\mU_k,\mV_{k+1}) +  	\left\langle \nabla_{\mV}g(\mU_{k},\mV_{k +1}), \mV_{k}-\mV_{k +1}\right\rangle \\
	&\quad + \int_0^1 \nabla^2_{\mV\mV} g(\mU_k,t\mV_k + (1-t)\mV_{k+1})[\mV_{k}-\mV_{k +1},\mV_{k}-\mV_{k +1}]\dif t\\
	& \geq g(\mU_k,\mV_{k+1}) + \frac{\lambda}{2}\|\mV_{k}-\mV_{k +1}\|_F^2,
	\end{align*}	
	which immediately implies
	\[
	g(\mU_k,\mV_k) - g(\mU_k,\mV_{k+1}) \geq \frac{\lambda}{2}\|\mV_{k}-\mV_{k +1}\|_F^2.
	\]
	
	Using similar argument, we have 
	\[
	g(\mU_k,\mV_{k+1}) - g(\mU_{k+1},\mV_{k+1})\geq \frac{\lambda}{2}\|\mU_{k}-\mU_{k +1}\|_F^2.
	\]
The proof is completed by summing the above two inequalities and recognizing that $\delta_{+}(\mU_k) = \delta_{+}(\mV_k) = \delta_{+}(\mU_{k+1}) = \delta_{+}(\mV_{k+1}) = 0$.
\end{proof}

\subsection{Proof of \Cref{thm:ANLS}}
\label{sec:prf thm ANLS}
\begin{lem}\label{lem:iterates regular}
Let $\{(\mU_k,\mV_k)\}$ be the sequence generated by \Cref{alg:ANLS}. Then the following holds. 
	\begin{enumerate}[(a)]
		\item The sequence $\{f(\mU_k,\mV_k)\}$ of function values is nonincreasing and it converges to some finite value:
		\[\lim\limits_{k\to\infty}f(\mU_k,\mV_k)=f^\star\]
		for some $f^\star\geq 0$.
		\item The difference between iterates sequence is convergent, i.e. 

		    \e
		    \lim_{k\rightarrow\infty} \| \mU_{k+1} -   \mU_{k} \|_F = 0,  \ \lim_{k\rightarrow\infty}\|   \mV_{k +1  } - \mV_{k}\|_F = 0.
		    \label{eq:difference converges}
		    \ee

	\end{enumerate}
\end{lem}

\begin{proof}[Proof of \Cref{lem:iterates regular}]
It follows from \Cref{lem:sufficient decrease} that
\[\sum_{k=0}^\infty f(\mW_k) -f(\mW_{k+1})\geq \frac{\lambda}{2}\sum_{k=0}^\infty \|\mW_{k+1}-\mW_{k}\|_F^2
\quad
\Longrightarrow
\quad
\sum_{k=0}^\infty \|\mW_{k+1}-\mW_{k}\|_F^2\le \frac{2f(\mW_0)}{\lambda}.\]
Now, we conclude the proof of (a) by identifying that the sequence $\{f(\mW_k)\}$ is non-increasing and  lower-bounded by zero. 
 For proving (b), we note that the sequence $\{\sum_{k=n}^\infty \|\mW_{k+1}-\mW_{k}\|_F^2\}_n$ is convergent, hence we are guaranteed that $\lim_{k\rightarrow\infty} \|\mW_{k+1}-\mW_{k}\|_F =0$.
\end{proof}

\begin{lem}\label{lem:bounded iterates}
The sequence $\{(\mU_k,\mV_k)\}$  generated by \Cref{alg:ANLS} lies in a bounded subset.
\end{lem}
\begin{proof}[Proof of \Cref{lem:bounded iterates}]
The proof is a direct consequence of \Cref{lem:bound:iterate} and the sufficient decrease property proved in \Cref{lem:sufficient decrease}.
\end{proof}

\begin{lem}\label{lem:safeguard}
Let $\{\mW_{k} = (\mU_k,\mV_k)\}$ be the sequence generated by \Cref{alg:ANLS}. Then there exist $\mS_{k+1} \in \partial_{\mU} f(\mU_{k+1},\mV_{k+1})$ and $\mD_{k+1} \in \partial_{\mV}  f(\mU_{k+1},\mV_{k+1})$  such that
\begin{align} \label{eq:safeguard}
	 \left\| 
	 \begin{bmatrix}\mS_{k+1}  \\ \mD_{k+1} \\ \end{bmatrix} \right\|_F  \leq (2B_0+\lambda+\|\mX\|_F)\|\mW_{k+1}-\mW_{k}\|_F.
\end{align}
\end{lem}
\begin{proof}[Proof of \Cref{lem:sufficient decrease}]
On one hand, by the definition of $\mV_{k+1}$, we have
\[\mzero\in \nabla_{\mV} g(\mU_{k},\mV_{k+1}) + \partial\sigma_{+}(\mV_{k+1}).\]
Along with the fact
\[\partial_{\mV} f(\mU_{k+1},\mV_{k+1})=\nabla_{\mV} g(\mU_{k+1},\mV_{k+1})+ \partial\sigma_{+}(\mV_{k+1}),\]
we denote by
\[\mD_{k+1}=\nabla_{\mV} g(\mU_{k+1},\mV_{k+1})-\nabla_{\mV} g(\mU_{k},\mV_{k+1})\in \partial_{\mV}  f(\mU_{k+1},\mV_{k+1}).\]
Then by the Lipschitz property of $g$ in \Cref{lem:lipchitz} and the boundedness property $\|\mU_k\|_F^2+\|\mV_k\|_F^2\leq B_0$ in \eqref{eqn:bound}, we have 
\[\|\mD_{k+1}\|_F\leq (2B_0+\lambda+\|\mX\|_F)\|\mU_{k+1}-\mU_{k}\|_F.\]

On the other hand, we let $\mS_{k+1}=\mzero$ which satisfies $\mS_{k+1} \in \partial_{\mU} f(\mU_{k+1},\mV_{k+1})$ since
\[\mU_{k+1}=\argmin_{\mU} f(\mU,\mV_{k+1}).\]

Thus, we have  $\|(\mS_{k+1},\mD_{k+1})\|_F\leq (2B_0+\lambda+\|\mX\|_F)\|\mW_{k+1}-\mW_{k}\|_F$.
 	 
\end{proof}

We denote $\calC(\mW_0)$ as the collection of the limit points  of  the sequence $\{\mW_k\}$ (which may depend on the initialization$(\mW_0)$). The following lemma demonstrate some useful property and  optimality of  $\calC(\mW_0)$.

\begin{prop}\label{prop:function value converges}
	Suppose the  sequence $\{(\mU_k,\mV_k)\}$ is generated by \Cref{alg:ANLS}. Then 
	\[
	   \lim\limits_{k\rightarrow\infty} f(\mU_k,\mV_k) = f(\mU^\star,\mV^\star), \ \ \forall \ (\mU^\star,\mV^\star) \in \calC(\mU_0,\mV_0).
	\]
And furthermore, 
	\[
	f(\mU_1^\star,\mV_1^\star) = f(\mU_2^\star,\mV_2^\star)
	\]
	 for any $(\mU_1^\star,\mV_1^\star)\in \calC(\mU_0,\mV_0)$ and $(\mU_2^\star,\mV_2^\star) \in \calC(\mU_0,\mV_0)$.  In other words, the sequence in function values converges to a critical value of \eqref{eq:SNMF by reg}.
\end{prop}
\begin{proof}[Proof of \Cref{prop:function value converges}]
	We first extract an arbitrary convergent subsequence $\{\mW_{k_m}\}_m$  which converges to $\mW^\star$. By the definition of the algorithm  we have 
	\[
	  \mU_k \geq 0, \ \mV_k \geq 0, \ \ \forall \ k\geq 0,
	\]
which implies that
	\[
	\mU^\star \geq 0, \ \mV^\star \geq 0.
	\]
Thus,
	\[
	    \lim\limits_{m\rightarrow \infty}  \delta_{+} (\mU_{k_m}) = 0, \ \lim\limits_{m\rightarrow \infty}  \delta_{+} (\mV_{k_m}) = 0.
	\]
We now take limit on subsequence 
	\[
	  \lim\limits_{m\rightarrow \infty}   f(\mW_{k_m}) = \lim\limits_{m\rightarrow \infty}  g(\mW_{k_m}) + \lim\limits_{m\rightarrow \infty}   ( \delta_{+} (\mU_{k_m}) +  \delta_{+} (\mV_{k_m}) ) =     g(\lim_{m\rightarrow \infty}\mW_{k_m}) = g(\mW^\star),
	\]
	where we have used the continuity of the smooth part $g(\mW)$ in \eqref{eq:SNMF by reg}.
	Then from  \Cref{lem:iterates regular} we know that  $\{f(\mW_k)\}$ forms a convergent sequence. The proof is completed by noting that for any convergent sequence, all its  subsequence must converge to the same limiting point.
\end{proof}

\begin{lem}\label{lem:limit points set}
	Suppose the sequence $\{\mW_k\}$ is generated by \Cref{alg:ANLS}. Then each element $\mW^\star=(\mU^\star,\mV^\star)\in \calC(\mW^0)$ is a critical point of \eqref{eq:SNMF by reg} and $\calC(\mW^0)$ is a nonempty, compact, and connected set, and satisfy
	\[\lim\limits_{k\to\infty}\dist(\mW_k,\calC(\mW_0))=0.\]
	
\end{lem}
\begin{proof}[Proof of \Cref{lem:limit points set}]
   	It follows from \Cref{lem:iterates regular} and \Cref{lem:safeguard} that there exist $\mS_{k+1} \in \partial_{\mU} f(\mU_{k+1},\mV_{k+1})$ and $\mD_{k+1} \in \partial_{\mV}  f(\mU_{k+1},\mV_{k+1})$  such that
   	$\|(\mS_{k+1},\mD_{k+1})\|_F\leq(2B_0+\lambda+\|\mX\|_F)\|\mU_{k+1}-\mU_{k}\|_F$
   	with the right hand side converging to zero as $k$ goes to infinity, i.e.,
   	\[
   	\lim\limits_{k\to\infty}(\mS_{k},\mD_{k})=\textbf{0}.
   	\]
   	Then we extract an arbitrary convergent subsequence $\{\mW_{k_m}\}_m$  with limit $\mW^\star$, i.e.,  $\lim_{m\rightarrow \infty} \mW_{k_m}= \mW^\star$.  Due to \Cref{lem:subdiff for f}, we have
   	\[
   	\mS_{k_m} = \nabla_{\mU}g(\mU_{k_m},\mV_{k_m}) + \overline\mS_{k_m}, \ \overline\mS_{k_m} \in \partial\delta_+(\mU_{k_m}).
   	\]
Since $\lim_{m \rightarrow \infty} \mS_{k_m} = \vzero,\lim_{m\rightarrow \infty} \mW_{k_m}= \mW^\star$, and $\nabla_{\mU}g$ is continuous, $\{\overline \mS_{k_m}\}$ is convergent. Denote by $\overline \mS^\star = \lim_{m \rightarrow \infty} \overline\mS_{k_m}$. By the definition of $\overline\mS_{k_m} \in \partial\delta_+(\mU_{k_m})$, for any $\mU'\in\R^{n\times r}$, we have
\[
\delta_+(\mU') - \delta_+(\mU_{k_m})\geq \langle \overline\mS_{k_m},\mU' - \mU_{k_m} \rangle.
\]
Since $\lim_{m\rightarrow \infty}  \delta_{+} (\mU_{k_m}) = \delta_{+} (\mU^\star) =0$, taking $m\rightarrow \infty$ for both sides of the above equation gives
\[
\delta_+(\mU') - \delta_+(\mU^\star)\geq \langle \overline\mS^\star,\mU' - \mU^\star \rangle.
\]
Since the above equation holds for any $\mU'\in\R^{n\times r}$, we have $\overline \mS^\star \in \partial \delta_+(\mU^\star)$ and thus $\vzero = \nabla_{\mU}g(\mU^\star,\mV^\star) + \overline\mS^\star\in \partial_{\mU}f(\mW^\star)$. With similar argument, we get $\vzero \in \partial_{\mV}f(\mW^\star)$ and thus
	\[
   	    \textbf{0} \in \partial f(\mW^\star),
   	\]
which implies that $\mW^\star$ is a critical point of \eqref{eq:SNMF by reg}.

	Finally, by \cite[Lemma 3.5]{bolte2014proximal} and identifying that the sequence $\{\mW_k\}$ is bounded and regular (i.e. $\lim_{k\to\infty}\|\mW_{k+1}-\mW_k\|_F=0$), we conclude that the set of accumulation points $\calC(\mW_0)$ is a nonempty compact and connect set satisfying
	\[\lim\limits_{k\to\infty}\dist(\mW_k,\calC(\mW_0))=0.\]
\end{proof}

\begin{lem}\label{lem:Uniform KL}
	For arbitrary $(\mU,\mV)\in \calC(\mU_0,\mV_0)$, we can uniformly find a set of constants $C_2>0, \delta>0, \theta\in[0,1)$ such that 
	\[
	\left|f(\mW) - f(\mW^\star)\right|^{\theta} \leq C_2 \dist(0, \partial f(\mW)).
	\]
	for all $(\mW)$ such that $\dist\left(\mW,\calC(\mW_0)\right) \leq \delta$.
\end{lem}
\begin{proof}[Proof of \Cref{lem:Uniform KL}]
It is easy and straightforward to identify that $f(\mU,\mV)$ satisfies the KL inequality at every point in its effective domain.  From \Cref{lem:limit points set} we know the set $\calC(\mW_0)$ is a compact and connected set.  Hence we can find finitely many balls $B(\mW_i,r_i)$  and their intersection to cover 
\[
   \calD   = \{ (\mW): \dist\left(\mW,\calC(\mW_0)\right) \leq \delta\},
\]
where each $r_i$ is chosen such that the KL inequality holds true at each center and we can choose  $c_i >0, \theta_i \in[0,1) $ that 
\[
     [f(\mW) - f(\mW_i)]^{\theta_i} \leq c_i \dist(0,\partial f(\mW)) , \ \ \forall \ \mW \in B(\mW_i,r_i).
\]
Hence it is straightforward to verify 
\[
\left|f(\mW) - f(\mW^\star)\right|^{\theta} \leq C_2 \dist(0, \partial f(\mW))
\]
for all $(\mW)$ such that $\dist\left(\mW,\calC(\mW_0)\right) \leq \delta$, where $C_2 = \max\{c_i\}$ and $\theta = \max\{\theta_i\}$.

 \end{proof}

We now prove \Cref{thm:ANLS}.  Due to \Cref{lem:limit points set} that  $\lim_{k\to\infty}\dist(\mW_k,\calC(\mW_0))=0$, for any fixed $\delta>0$ there exists $k_0$ such that  $\dist\left(\mW_k,\calC(\mW_0)\right) \leq \delta$ for all $k\geq k_0$. Hence 
\[
      \left|f(\mW_k) - f(\mW^\star)\right|^\theta \leq  C_2 \dist(0, \partial f(\mW_k)), \ \ \forall k\geq k_0.
\]
In the subsequent analysis, we restrict to $k\geq k_0$. We now construct a concave function  $x^{1-\theta}$ with domain $x> 0$, $x_1^{1-\theta} \leq x_2^{1-\theta} + (1-\theta) x_2^{-\theta}(x_1-x_2), \forall x_1>0,x_2>0$, where recall that $\theta \in [0,1)$. Replacing $x_1$ by $f(\mW_{k+1}) - f(\mW^\star)$ and  $x_2$ by $f(\mW_{k}) - f(\mW^\star)$ gives 
\[
\left(f(\mW_{k}) - f(\mW^\star)\right)^{1-\theta} - \left(f(\mW_{k+1}) - f(\mW^\star )\right)^{1-\theta} \geq  (1-\theta)\frac{f(\mW_{k}) - f(\mW_{k+1}) }{\left(f(\mW_{k}) - f(\mW^\star)\right)^\theta}.
\]
which togerher with \Cref{lem:sufficient decrease} and \Cref{lem:Uniform KL} gives
\begin{align*}
   & \left(f(\mW_{k}) - f(\mW^\star)\right)^{1-\theta} - \left(f(\mW_{k+1}) - f(\mW^\star )\right)^{1-\theta}  \\
   &\geq  \frac{\lambda(1-\theta)}{2C_2} \frac{\|\mW_{k} - \mW_{k+1}\|_F^2 }{\dist(0,\partial f(\mW_k))},\\
    &\geq \frac{\lambda(1-\theta)}{2C_2(2B_0+\lambda+\|\mX\|_F)} \frac{\|\mW_{k} - \mW_{k+1}\|_F^2 }{\|\mW_{k} - \mW_{k-1}\|_F},
\end{align*}
where we use \Cref{lem:safeguard} in the last inequality. 

By construction we have 
\begin{align*}
    &\frac{\|\mW_{k} - \mW_{k+1}\|_F^2 }{\|\mW_{k} - \mW_{k-1}\|_F} + \|\mW_{k} - \mW_{k-1}\|_F - \|\mW_{k} - \mW_{k-1}\|_F\\
    &\geq 2 \|\mW_{k} - \mW_{k+1}\|_F - \|\mW_{k} - \mW_{k-1}\|_F.
\end{align*}
Combing the above two inequalities provides 
\[
      2 \|\mW_{k} - \mW_{k+1}\|_F - \|\mW_{k} - \mW_{k-1}\|_F \leq \beta  \left(f(\mW_{k}) - f(\mW^\star)\right)^{1-\theta} - \left(f(\mW_{k+1}) - f(\mW^\star )\right)^{1-\theta}, 
\]
where with $\beta >0$ is some constant depending on $\lambda,\theta, C_2, B_0$, and $\|\mX\|_F$. Repeating the above inequality and summing them up from $\widetilde k$ (which is larger than $k_0$) to $m$, then taking limit that $m\rightarrow \infty$ yields 
\e \label{eq:difference summable}
     \sum_{k=\widetilde k}^{\infty}  \|\mW_{k} - \mW_{k+1}\|_F  \leq  \|\mW_{\widetilde k} - \mW_{\widetilde k-1}\|_F + \beta  \left(f(\mW_{\widetilde k}) - f(\mW^\star)\right)^{1-\theta},
\ee 
where we invoke the fact that $f(\mW_{k})\rightarrow f(\mW^\star )$. Hence 
\[
     \sum_{k=k_0}^{\infty}  \|\mW_{k} - \mW_{k+1}\|_F  < \infty.
\]
Following some standard arguments one can see that  
\[
   \limsup_{t\rightarrow \infty, t_1,t_2\geq t}  \|\mW_{t_1} - \mW_{t_2}\|_F  = 0,
\]
which implies the sequence $\{  \mW_k \}$  is Cauchy, hence a convergent sequence.  The limit point set $\calC(\mW_0)$ is  singleton $\mW^\star$, and  from \Cref{lem:limit points set} it is guaranteed to be one critical point of \eqref{eq:SNMF by reg}, i.e. 
	\[
\lim\limits_{k\rightarrow \infty} (\mU_k,\mV_k) = (\mU^\star,\mV^\star)
\]
where $(\mU^\star,\mV^\star)$ is a critical point of \eqref{eq:SNMF by reg}.

As for convergence rate, we can see from \eqref{eq:difference summable} and triangle inequality that 
\e\label{eq:convergence rate 1}
     \|\mW_{\widetilde k} - \mW^\star\|_F \leq\sum_{k=\widetilde k}^{\infty}  \|\mW_{k} - \mW_{k+1}\|_F  \leq   \|\mW_{\widetilde k} - \mW_{\widetilde k-1}\|_F + \beta  \left(f(\mW_{\widetilde k}) - f(\mW^\star)\right)^{1-\theta},
\ee
from which we observe that the convergence rate of $\mW_{\widetilde k} \rightarrow  \mW^\star$  is at least as fast as the speed that $ \|\mW_{\widetilde k} - \mW_{\widetilde k-1}\|_F + \beta  \left(f(\mW_{\widetilde k}) - f(\mW^\star)\right)^{1-\theta}$ tends to 0.   \Cref{lem:Uniform KL}  and \Cref{lem:safeguard} provide the bound $\beta \left(f(\mW_{\widetilde k}) - f(\mW^\star)\right)^{1-\theta} \leq \alpha \|\mW_{\widetilde k} - \mW_{\widetilde k-1}\|_F^{\frac{1-\theta}{\theta}} $. Then, we have
\[
    \sum_{k=\widetilde k}^{\infty}  \|\mW_{k} - \mW_{k+1}\|_F \leq \|\mW_{\widetilde k} - \mW_{\widetilde k-1}\|_F +\alpha \|\mW_{\widetilde k} - \mW_{\widetilde k-1}\|_F^{\frac{1-\theta}{\theta}}. 
\]
We divide the following analysis into two cases based on the value of the KL exponent $\theta$. 

\emph{Case I}: $\theta\in [0,\frac{1}{2}]$. This case means $\frac{1-\theta}{\theta} \geq 1$. We define $P_{\widetilde k} = \sum_{i = \widetilde k}^\infty \|\mW_{i+1} - \mW_i\|_F$,
\e 
P_{\widetilde k} \leq P_{{\widetilde k}-1}- P_{\widetilde k} + \alpha \left[P_{\widetilde k-1}- P_{\widetilde k}\right]^{\frac{1-\theta}{\theta}}. \label{eq:convergence rate 2}
\ee
Since $P_{\widetilde k-1}- P_{\widetilde k} \rightarrow 0$, there exists a positive integer $k_1$ such that $P_{{\widetilde k}-1}- P_{\widetilde k} < 1,~\forall~\widetilde k\geq k_1$.    Thus,
\[
P_{\widetilde k} \leq \left(1+ \alpha\right) (P_{\widetilde k-1}- P_{\widetilde k}),~~~~\forall ~ \widetilde k\geq \max\{ k_0,k_1\},
\]
which implies that
\e
P_{\widetilde k} \leq \rho \cdot P_{\widetilde k-1},~~~~\forall ~ \widetilde k\geq \max\{ k_0,k_1\},
\label{eq:Pk decay}\ee
where $\rho = \frac{1+\alpha}{2+ \alpha} \in (0,1)$. This together with  \eqref{eq:convergence rate 1} gives the linear convergence rate
\e \label{eq:linear convergence}
        \|\mW_{ k} - \mW^\star\|_F  \leq \calO( \rho^{k-\overline k} ), \ \forall \ k\geq \overline k. 
\ee
where $\overline k = \max\{ k_0,k_1\}$.

\emph{Case II}: $\theta\in (1/2,1)$. This case means $\frac{1-\theta}{\theta} \leq 1$. Based on the former results, we have 
\[
P_{\widetilde k} \leq \left(1+ \alpha\right) \left[P_{\widetilde k-1}- P_{\widetilde k}\right]^{\frac{1-\theta}{\theta}},~~~~\forall ~ \widetilde k\geq \max\{ k_0,k_1\}.
\]
We now run into the same situation as in \cite[Theorem 2]{bolte2014proximal} (after equation (13)) and  \cite[Theorem 2]{zhu2018convergence} (after equation (30)), hence  following  a similar argument  gives 
 \[
 P_{\widetilde k}^{\frac{1-2\theta}{1-\theta}} -  P_{\widetilde k-1}^{\frac{1-2\theta}{1-\theta}}  \geq \zeta, \ \forall \ k\geq \overline k,
 \]
 for some $\zeta >0$. Repeating and summing up the above inequality from $\overline k = \max\{ k_0,k_1\}$ to any $k> \overline k$, we have
 	\[
P_{\widetilde k}\leq \left[ P_{\widetilde k-1}^{\frac{1-2\theta}{1-\theta}} + \zeta(\widetilde k-\overline k)  \right]^{-\frac{1-\theta}{2\theta -1}} = \calO\left((\widetilde k-\overline k) ^{-\frac{1-\theta}{2\theta -1}} \right), \ \forall \ \widetilde k> \overline k.
 \]
 Finally, the following sublinear convergence holds 
 \e \label{eq:sublinear convergence}
 \|\mW_{ k} - \mW^\star\|_F  \leq \calO\left(( k-\overline k) ^{-\frac{1-\theta}{2\theta -1}} \right), \ \forall \  k> \overline k.
 \ee
 We end this proof by commenting that both linear and sublinear convergence rate are  closely related to  the KL exponent $\theta$ at the critical point $\mW^\star$.

 \subsection{Proof sketch of  \Cref{lem:sufficient decrease for HALS} and \Cref{thm:HALS}}
\label{sec:prf HALS} 
Note that both SymHALS and SymANLS share the same algorithmic framework, i.e., alternating minimization. The only difference is that SymHALS has multiple optimization variables while SymANLS has only two variables. Thus, \Cref{lem:sufficient decrease for HALS} and \Cref{thm:HALS} can be proved with similar arguments used for \Cref{lem:sufficient decrease} and \Cref{thm:ANLS}. For example, with a similar argument displayed in \Cref{lem:sufficient decrease}, the iterates in SymHALS satisfies
\begin{itemize}
	\item \emph{Case I}: by updating $\vu_i$, i.e. $(\vu_1^{k+1},\cdots,\vu_{i-1}^{k+1},\vu_i^k,\cdots,\vu_r^k, \mV_k) \rightarrow (\vu_1^{k+1},\cdots,\vu_i^{k+1},\vu_{i+1}^{k},\cdots,\vu_r^k, \mV_k)$, we have 
	\[
	f(\vu_1^{k+1},\cdots,\vu_i^k,\cdots,\vu_r^k, \mV_k) - f(\vu_1^{k+1},\cdots,\vu_i^{k+1},\cdots,\vu_r^k,\mV_{k+1})
	 \geq \frac{\lambda}{2}\|\vu_i^{k+1}  - \vu_i^{k}\|_2^2. 
	\]
	\item \emph{Case II}: By updating $\vv_i$, i.e. $(\mU_{k+1}, \vv_1^{k+1},\cdots,\vv_{i-1}^{k+1},\vv_i^k,\cdots,\vv_r^k) \rightarrow (\mU_{k+1}, \vv_1^{k+1},\cdots,\vv_i^{k+1},\vv_{i+1}^{k},\cdots,\vv_r^k)$, we have 
	\[
	f(\mU_{k+1}, \vv_1^{k+1},\cdots,\vv_{i-1}^{k+1},\vv_i^k,\cdots,\vv_r^k) - f(\mU_{k+1}, \vv_1^{k+1},\cdots,\vv_i^{k+1},\vv_{i+1}^{k},\cdots,\vv_r^k)
	\geq \frac{\lambda}{2}\|\vv_i^{k+1}  - \vv_i^{k}\|_2^2. 
	\]
\end{itemize}
Unrolling the update from $\vu_1$ to $\vv_r$ and summing them up, we get the same descend inequality in \Cref{lem:sufficient decrease for HALS}. 
By the same strategy for \Cref{lem:safeguard}, one can then prove \Cref{thm:HALS} following the same argument in \Cref{thm:ANLS}.

\section*{Acknowledgment}
The authors thank Dr. Songtao Lu for sharing the code used in \cite{lu2017nonconvex} and
the three anonymous reviewers as well as the area chair for their constructive comments.

\bibliographystyle{ieeetr}
\bibliography{Convergence}

\begin{thebibliography}{10}

\bibitem{lee1999learning}
D.~D. Lee and H.~S. Seung, ``Learning the parts of objects by non-negative
  matrix factorization,'' {\em Nature}, vol.~401, no.~6755, p.~788, 1999.

\bibitem{guillamet2002non}
D.~Guillamet and J.~Vitria, ``Non-negative matrix factorization for face
  recognition,'' in {\em Topics in artificial intelligence}, pp.~336--344,
  Springer, 2002.

\bibitem{shahnaz2006document}
F.~Shahnaz, M.~W. Berry, V.~P. Pauca, and R.~J. Plemmons, ``Document clustering
  using nonnegative matrix factorization,'' {\em Information Processing \&
  Management}, vol.~42, no.~2, pp.~373--386, 2006.

\bibitem{ma2014signal}
W.-K. Ma, J.~M. Bioucas-Dias, T.-H. Chan, N.~Gillis, P.~Gader, A.~J. Plaza,
  A.~Ambikapathi, and C.-Y. Chi, ``A signal processing perspective on
  hyperspectral unmixing: Insights from remote sensing,'' {\em IEEE Signal
  Processing Magazine}, vol.~31, no.~1, pp.~67--81, 2014.

\bibitem{gillis2014and}
N.~Gillis, ``The why and how of nonnegative matrix factorization,'' {\em
  Regularization, Optimization, Kernels, and Support Vector Machines}, vol.~12,
  no.~257, 2014.

\bibitem{lee2001algorithms}
D.~D. Lee and H.~S. Seung, ``Algorithms for non-negative matrix
  factorization,'' in {\em Advances in neural information processing systems},
  pp.~556--562, 2001.

\bibitem{lin2007projected}
C.-J. Lin, ``Projected gradient methods for nonnegative matrix factorization,''
  {\em Neural computation}, vol.~19, no.~10, pp.~2756--2779, 2007.

\bibitem{kim2008toward}
J.~Kim and H.~Park, ``Toward faster nonnegative matrix factorization: A new
  algorithm and comparisons,'' in {\em International Conference on Data
  Mining}, pp.~353--362, 2008.

\bibitem{cichocki2009fast}
A.~Cichocki and A.-H. Phan, ``Fast local algorithms for large scale nonnegative
  matrix and tensor factorizations,'' {\em IEICE transactions on fundamentals
  of electronics, communications and computer sciences}, vol.~92, no.~3,
  pp.~708--721, 2009.

\bibitem{kuang2015symnmf}
D.~Kuang, S.~Yun, and H.~Park, ``Symnmf: nonnegative low-rank approximation of
  a similarity matrix for graph clustering,'' {\em Journal of Global
  Optimization}, vol.~62, no.~3, pp.~545--574, 2015.

\bibitem{ding2005equivalence}
C.~Ding, X.~He, and H.~D. Simon, ``On the equivalence of nonnegative matrix
  factorization and spectral clustering,'' in {\em Proceedings ofInternational
  Conference on Data Mining}, pp.~606--610, 2005.

\bibitem{he2011symmetric}
Z.~He, S.~Xie, R.~Zdunek, G.~Zhou, and A.~Cichocki, ``Symmetric nonnegative
  matrix factorization: Algorithms and applications to probabilistic
  clustering,'' {\em IEEE Transactions on Neural Networks}, vol.~22, no.~12,
  pp.~2117--2131, 2011.

\bibitem{tu2015low}
S.~Tu, R.~Boczar, M.~Soltanolkotabi, and B.~Recht, ``Low-rank solutions of
  linear matrix equations via procrustes flow,'' {\em arXiv preprint
  arXiv:1507.03566}, 2015.

\bibitem{zhu2017global}
Z.~Zhu, Q.~Li, G.~Tang, and M.~B. Wakin, ``The global optimization geometry of
  low-rank matrix optimization,'' {\em arXiv preprint arXiv:1703.01256}, 2017.

\bibitem{vandaele2016efficient}
A.~Vandaele, N.~Gillis, Q.~Lei, K.~Zhong, and I.~Dhillon, ``Efficient and
  non-convex coordinate descent for symmetric nonnegative matrix
  factorization,'' {\em IEEE Transactions on Signal Processing}, vol.~64,
  no.~21, pp.~5571--5584, 2016.

\bibitem{lu2017nonconvex}
S.~Lu, M.~Hong, and Z.~Wang, ``A nonconvex splitting method for symmetric
  nonnegative matrix factorization: Convergence analysis and optimality,'' {\em
  IEEE Transactions on Signal Processing}, vol.~65, no.~12, pp.~3120--3135,
  2017.

\bibitem{attouch2010proximal}
H.~Attouch, J.~Bolte, P.~Redont, and A.~Soubeyran, ``Proximal alternating
  minimization and projection methods for nonconvex problems: An approach based
  on the kurdyka-lojasiewicz inequality,'' {\em Mathematics of Operations
  Research}, vol.~35, no.~2, pp.~438--457, 2010.

\bibitem{huang2016flexible}
K.~Huang, N.~D. Sidiropoulos, and A.~P. Liavas, ``A flexible and efficient
  algorithmic framework for constrained matrix and tensor factorization,'' {\em
  IEEE Transactions on Signal Processing}, vol.~64, no.~19, pp.~5052--5065,
  2016.

\bibitem{razaviyayn2013unified}
M.~Razaviyayn, M.~Hong, and Z.-Q. Luo, ``A unified convergence analysis of
  block successive minimization methods for nonsmooth optimization,'' {\em SIAM
  J. Optimization}, vol.~23, no.~2, pp.~1126--1153, 2013.

\bibitem{gillis2012accelerated}
N.~Gillis and F.~Glineur, ``Accelerated multiplicative updates and hierarchical
  als algorithms for nonnegative matrix factorization,'' {\em Neural
  computation}, vol.~24, no.~4, pp.~1085--1105, 2012.

\bibitem{hsieh2011fast}
C.-J. Hsieh and I.~S. Dhillon, ``Fast coordinate descent methods with variable
  selection for non-negative matrix factorization,'' in {\em Proceedings of the
  17th ACM SIGKDD international conference on Knowledge discovery and data
  mining}, pp.~1064--1072, ACM, 2011.

\bibitem{xu2003document}
W.~Xu, X.~Liu, and Y.~Gong, ``Document clustering based on non-negative matrix
  factorization,'' in {\em International ACM SIGIR conference on Research and
  development in informaion retrieval}, pp.~267--273, 2003.

\bibitem{bolte2014proximal}
J.~Bolte, S.~Sabach, and M.~Teboulle, ``Proximal alternating linearized
  minimization for nonconvex and nonsmooth problems,'' {\em Mathematical
  Programming}, vol.~146, no.~1-2, pp.~459--494, 2014.

\bibitem{rockafellar2015convex}
R.~T. Rockafellar, {\em Convex analysis}.
\newblock Princeton university press, 2015.

\bibitem{bolte2007lojasiewicz}
J.~Bolte, A.~Daniilidis, and A.~Lewis, ``The {\l}ojasiewicz inequality for
  nonsmooth subanalytic functions with applications to subgradient dynamical
  systems,'' {\em SIAM Journal on Optimization}, vol.~17, no.~4,
  pp.~1205--1223, 2007.

\bibitem{attouch2009convergence}
H.~Attouch and J.~Bolte, ``On the convergence of the proximal algorithm for
  nonsmooth functions involving analytic features,'' {\em Mathematical
  Programming}, vol.~116, no.~1, pp.~5--16, 2009.

\bibitem{attouch2013convergence}
H.~Attouch, J.~Bolte, and B.~F. Svaiter, ``Convergence of descent methods for
  semi-algebraic and tame problems: proximal algorithms, forward--backward
  splitting, and regularized gauss--seidel methods,'' {\em Mathematical
  Programming}, vol.~137, no.~1-2, pp.~91--129, 2013.

\bibitem{zhu2018convergence}
Z.~Zhu and X.~Li, ``Convergence analysis of alternating nonconvex
  projections,'' {\em arXiv preprint arXiv:1802.03889}, 2018.

\end{thebibliography}

\end{document}